\pdfoutput=1
\def \myarxiv {}

\documentclass{article}

\usepackage{microtype}
\usepackage{graphicx}
\usepackage{booktabs} %

\usepackage{algorithm}
\usepackage{algorithmic}

\usepackage{amsfonts}       %
\usepackage{amsmath} %
\usepackage{caption,subcaption}
\usepackage{wrapfig}
\graphicspath{{figures/}} 
\usepackage{hyperref}

\usepackage[accepted]{icml2018}

\icmltitlerunning{Composite Functional Gradient Learning of Generative
  Adversarial Models}

\def\<#1,#2>{\langle #1,#2\rangle}

\newtheorem{lemma}{Lemma}[section]

\newtheorem{theorem}{Theorem}[section]

\newtheorem{assumption}{Assumption}[section]

\makeatletter
\newcommand\tightpara{\@startsection{paragraph}{4}{\z@}{0.5ex plus
   0ex minus 0ex}{-1em}{\normalsize\bf}}
\newcommand\normalpara{\@startsection{paragraph}{4}{\z@}{1.5ex plus
   0.5ex minus .2ex}{-1em}{\normalsize\bf}}    
\makeatother 
\newcommand{\mypara}{\tightpara}

\newcommand{\tightDisplayBegin}[1]{\begingroup \setlength{\belowdisplayskip}{#1} \setlength{\belowdisplayshortskip}{#1} \setlength{\abovedisplayskip}{#1} \setlength{\abovedisplayshortskip}{#1}}
\newcommand{\tightDisplayEnd}{\endgroup}
\begin{document}

\twocolumn[
\icmltitle{Composite Functional Gradient Learning of Generative
  Adversarial Models}

\icmlsetsymbol{equal}{*}

\begin{icmlauthorlist}
\icmlauthor{Rie Johnson}{equal,rj}
\icmlauthor{Tong Zhang}{equal,tz}
\end{icmlauthorlist}

\icmlaffiliation{rj}{RJ Research Consulting, Tarrytown, NY, USA}
\icmlaffiliation{tz}{Tencent AI Lab, Shenzhen, China}

\icmlcorrespondingauthor{Rie Johnson}{riejohnson@gmail.com}

\vskip 0.3in
]

\printAffiliationsAndNotice{\icmlEqualContribution} %

\newcommand{\Real}{{\mathbb{R}}}
\newcommand{\datadim}{r} %
\newcommand{\dataspace}{\Real^\datadim} 
\newcommand{\rvData}{X} %
\newcommand{\rvZ}{Z}    %
\newcommand{\real}{{x^*}}
\newcommand{\realidx}[1]{{x^*_{#1}}}

\newcommand{\calA}{\mathcal{A}}
\newcommand{\idealD}{{\mathcal{D}}}
\newcommand{\Deps}{{D_{\epsilon}}}

\newcommand{\makegen}[1]{{#1}}  %
\newcommand{\makenext}[1]{{{#1}'}}
\newcommand{\gen}{\makegen{x}}
\newcommand{\rvGen}{\makegen{\rvData}} %

\newcommand{\realp}{p_*} 
\newcommand{\genp}{\makegen{p}}

\newcommand{\nextgen}{\makenext{\gen}}
\newcommand{\nextgenp}{\makenext{\genp}}
\newcommand{\rvNextgen}{\makenext{\rvGen}}  %
\newcommand{\logd}{log$d$}

\newcommand{\ganD}{d} 
\newcommand{\D}{\ganD}
\newcommand{\justx}{x}
\newcommand{\any}{x}

\newcommand{\orgAlg}{CFG}
\newcommand{\ourAlg}{ICFG} 
\newcommand{\xourAlg}{xICFG}
\newcommand{\capfont}{\small}
\newcommand{\prm}[1]{\theta_{#1}}
\newcommand{\justG}{G}
\newcommand{\Gidx}[1]{\justG_{#1}}
\newcommand{\gidx}[1]{g_{#1}}
\newcommand{\etaidx}[1]{\eta_{#1}}
\newcommand{\justD}{D}
\newcommand{\Didx}[1]{\justD_{#1}}
\newcommand{\realset}{S_*}
\newcommand{\inpt}{z}
\newcommand{\prior}{p_\inpt}
\newcommand{\xprior}{q_\inpt}
\newcommand{\inptPl}{S_z} %
\newcommand{\inptPlSz}{|\inptPl|}         %
\newcommand{\minib}{b}
\newcommand{\ganG}{G}
\newcommand{\dupd}{U} %
\newcommand{\xG}{\widetilde{G}}  %

\begin{abstract}
This paper first presents a theory for generative adversarial methods 
that does not rely on the traditional minimax formulation.  
It shows that with a strong discriminator, a good generator can be learned so that 
the KL divergence between the distributions of real data and generated data improves 
after each functional gradient step
until it converges to zero. 
Based on the theory, we propose a new stable generative adversarial method. 
A theoretical insight into the original GAN from this new viewpoint is also provided. 
The experiments on image generation show the effectiveness of our new method.  
\end{abstract}

\section{Introduction}
We consider observed real data $\realidx{1},\ldots, \realidx{n} \in \dataspace$ from an unknown
distribution $p_*$ on $\dataspace$.
Moreover, assume that we are given a random variable $\rvZ$ with a known distribution 
such as a Gaussian. %
We are interested in learning a random variable transformation $G(\rvZ)$ so that the
generated data $G(\rvZ)$
has a probability density function that is close to the real
distribution $p_*$.
This is the setting considered in {\em generative adversarial networks} 
(GAN) \cite{GAN14}, and the transformation
$G(\rvZ)$ is often referred to as a {\em generator}. 
While GAN has been widely used, it is also known that GAN is 
difficult to train due to its instability, 
which has led to numerous studies, e.g., 
Wasserstein GAN (WGAN) and its extensions to pursue a different minimax objective
\cite{WGAN17,WGANgp17,WGANfs17},
mode-regularized GAN to tackle the issue of mode collapse \cite{CLJBL17}, 
unrolled GAN \cite{unrollGAN17}, 
AdaGAN \cite{ADAGAN17}, 
MMD GAN \cite{MMDGAN17}, 
and references therein. 

An important concept introduced by GAN is the idea of {\em adversarial
  learner}, denoted here by $\D$, which tries to discriminate 
real data from generated data. Mathematically, GAN solves
the following minimax optimization problem:
\tightDisplayBegin{0.5pt}
\begingroup
\setlength{\thinmuskip}{1mu}
\setlength{\medmuskip}{0mu}
\setlength{\thickmuskip}{0mu}
\begin{equation}
\max_\D \min_G \left[ \sum_{\real\in\text{real data}} \ln \D(\real)
+ \sum_{G(z)\in\text{fake data}} \ln (1- \D(G(z))) 
\right] . \label{eq:gan}
\end{equation}
\endgroup
\tightDisplayEnd
Parameterizing $\D$ and $G$, 
\eqref{eq:gan}
can be viewed as a saddle point problem in optimization, which can be
solved using a stochastic gradient method, where one takes a gradient
step with respect to the parameters in $\D$ and $G$ (see
Algorithm~\ref{alg:gan} below).  
However, the practical procedure, suggested by the original work \cite{GAN14}, 
replaces minimization of $\log(1-\D(G(z)))$ with respect to $G$ in \eqref{eq:gan}
with maximization of $\log(\D(G(z)))$ with respect to $G$, 
called the {\em \logd\ trick}.  
Thus, GAN with the \logd\ trick, though often more effective, can
{\em not directly} 
be explained by the theory based on the minimax formulation \eqref{eq:gan}. 

This paper presents 
a new 
theory
for generative adversarial methods
which does not rely on the minimax formulation \eqref{eq:gan}. 
Our theory shows that one can learn a good generator $G(\rvZ)$ 
where `goodness' is measured by the KL-divergence between the distributions of real data and generated data, 
by
using {\em functional gradient learning} greedily, similar to gradient
boosting \cite{frie:01}.  
However, unlike the standard gradient boosting, which uses
additive models, we consider functional compositions in the
following form
\begin{equation}
G_{t}(\rvZ) = G_{t-1}(\rvZ) + \eta_t g_t(G_{t-1}(\rvZ)),~(t=1,\ldots,T) \label{eq:multi-step}
\end{equation}
to obtain 
$G(\rvZ) = G_{T}(\rvZ)$.
Here $\eta_t$ is a small learning rate, and each $g_t$ is a function to
be estimated from data.  An initial generator $G_0(\rvZ) \in \dataspace$ 
is assumed to be given.  
We learn 
from data 
$g_t$ greedily from $t=1$ to $t=T$
so that improvement (in terms of the KL-divergence) is guaranteed.  

Our theory leads to a new stable generative adversarial method.  
It also provides a new theoretical insight into the original GAN both with and without the \logd\ trick. 
The experiments show the effectiveness of our new method 
on image generation in comparison with GAN variants.  

\mypara{Notation} 
Throughout the paper, 
we use $\any$ to denote data in $\dataspace$, and in particular, we use 
$\real$ to denote real data.  
The probability density function of real data is denoted by $p_*$.  
We use $\|\cdot\|$ to denote the vector 2-norm and 
$\nabla h(\any)$ to denote the gradient w.r.t. $\any$
of a scalar function $h(\any)$. 
\section{Theory}

\begin{algorithm*}[ht]
  \caption{\orgAlg: Composite Functional Gradient Learning of GAN} 
  \label{alg:fg}
\begin{algorithmic}[1]
\REQUIRE real data $\realidx{1},\ldots,\realidx{n}$, 
         initial generator $G_0(z)$ with 
         generated data $\{G_0(z_1),\ldots, G_0(z_m)\}$.  
         Meta-parameter: $T$. 
\FOR {$t=1,2,\ldots,T$}
\STATE $D_{t}(x) \leftarrow \arg\min_D \left[ 
       \frac{1}{n}\sum_{i=1}^n \ln(1+\exp(-D(\realidx{i}))) + \frac{1}{m}\sum_{i=1}^m \ln(1+\exp(D(G_{t-1}(z_i))))
       \right]$
       \label{lno:cfgopt}
\STATE $g_{t}(x) \leftarrow s_t(x) \nabla D_{t}(x)$ $~~~$($s_t(x)$ is for scaling, e.g., most simply $s_t(x)=1$) \label{line:cfg}
\STATE  $G_{t}(z) \leftarrow  G_{t-1}(z) + \eta_{t} g_{t}(G_{t-1}(z))$, for some $\eta_{t}>0$.
\ENDFOR
\RETURN generator $G_{T}(z)$
\end{algorithmic}
\end{algorithm*}

To present our theory, 
we start with stating assumptions.  %
We then analyze one step of random variable transformation in \eqref{eq:multi-step} 
(i.e., transforming $G_{t-1}(\rvZ)$ to $G_{t}(\rvZ)$) 
and examine an algorithm suggested by this analysis. %

\subsection{Assumptions and definitions}
\label{sec:assump}

\newcommand{\localD}{D'}
\newcommand{\genset}{S}

\mypara{A strong discriminator} 
Given a set $\realset$ of real data and a set $\genset$ of generated data, 
assume that we can obtain a strong {\em discriminator} $D$ 
using logistic regression so that 
$
D \approx 
$
\tightDisplayBegin{0.5pt}
\begingroup
\setlength{\thinmuskip}{0mu}
\setlength{\medmuskip}{0mu}
\setlength{\thickmuskip}{0mu}
\begin{align*}
  \arg\min_{\localD} \left[  \frac{1}{|\realset|}\sum_{\justx \in \realset} \ln (1 + e^{-\localD(\justx)} ) + 
                             \frac{1}{|\genset|} \sum_{\justx \in \genset}  \ln (1 + e^{ \localD(\justx)} ) \right]~. 
\end{align*}
\endgroup
\tightDisplayEnd
$D$ tries to discriminate the real data from the generated data. 
Here we use the logistic model, and so $\D$ in \eqref{eq:gan} corresponds to
$\frac{1}{1+\exp(-D(x))}$. 
Define a quantity $\idealD(\any)$ by 
\tightDisplayBegin{3pt}
\[
\idealD(\any) :=\ln \frac{ p_*(\any) }{ \genp(\any) }
\] 
\tightDisplayEnd
where $p_*$ and $p$ are the probability density functions of real data and generated data, 
respectively, 
and assume that $|\idealD(\any)|$$<$$\infty$.  
$p_*$ and $p$ are thus assumed to be nonzero everywhere. 
When the number of given examples is sufficiently large,  
the standard statistical consistency theory of logistic regression 
(see e.g., \cite{Z04}) 
implies that 
$
D(x) \approx \idealD(\any)~. %
$
Therefore, 
assume that 
the following $\epsilon$-approximation condition is satisfied
for a small $\epsilon>0$: 
\tightDisplayBegin{5pt}
\begin{align*}
\int &q_*(\any) \left(\left|D(\any)-\idealD(\any)\right|+\left|e^{D(\any)}-e^{\idealD(\any)}\right|\right) d \any \le \epsilon, \\[0pt]
     &q_*(\any)=p_*(\any) \max(1,\|\nabla \ln p_*(\any)\|)~. \\
\end{align*}
\tightDisplayEnd
Note that the assumption of the {\em optimal} discriminator has been commonly used, 
and we slightly relax it to a {\em strong} discriminator by quantifying the deviation 
from the optimum by $\epsilon$.  
%
%
%
\mypara{Smooth and light-tailed $p_*$} 
Assume that $p_*$, the density of real data, is smooth with light tails; 
we use a constant $h_0>0$ that depends on the shape of $p_*$.  
Due to the space limit, the precise statements are deferred to the Appendix.
Common exponential distributions such as
Gaussian distributions and mixtures of Gaussians all satisfy the
assumption, and an arbitrary distribution can always be approximated 
to an arbitrary precision 
by a mixture of Gaussians.  

\subsection{Analyzing one step of random variable transformation}
\label{sec:onestep}
The goal is to approximate the true density $p_*$ on $\dataspace$ 
through \eqref{eq:multi-step}.  
Our analysis here focuses on one step of \eqref{eq:multi-step} at time $t$, namely, 
random variable transformation of $G_{t-1}(\rvZ)$ to $G_{t}(\rvZ)$.  
To simplify notation, we assume that 
we are given a random variable $\rvGen$ with a probability density 
$\genp$ on $\dataspace$. We are interested in finding a function
$g: \dataspace \to \dataspace$, so that the 
transformed variable $\rvNextgen=  \rvGen + \eta g(\rvGen)$
for small $\eta>0$
has a distribution closer to $p_*$. We show
that this can be achieved with 
a gradient-like step in the function space. 

\newcommand{\ourKL}{L}
To measure the distance of a density $p$ from the true density $p_*$, we will keep track
of the KL-divergence %
\tightDisplayBegin{5pt}
\begin{equation}
\ourKL(p) = \int p_*(x) \ln \frac{p_*(x)}{p(x)} d x  \label{eq:close}
\end{equation}
\tightDisplayEnd
before and after the transformation.
We know that $\ourKL(p) \geq 0$ for all $p$, and $\ourKL(p)=0$ if and only if
$p=p_*$. 

The following theorem consequently shows that with an appropriately chosen $g(\cdot)$,
the transformation $\rvGen \to \rvGen + \eta g(\rvGen)$ can always reduce
the KL-divergence $\ourKL(\cdot)$. This means that transformation $\rvGen + \eta
g(\rvGen)$ is an improvement from $\rvGen$.
The proof is 
given in the Appendix. 
\begin{theorem}
Under the assumptions in Section \ref{sec:assump}, 
let $g: \dataspace \to \dataspace$ be a  continuously
differentiable transformation such that $\|g(\cdot)\| \leq a$ and 
$\|\nabla g(\cdot)\| \leq b$.
Let $\genp$ be the probability density of a random variable $\rvGen$, and 
let $\nextgenp$ be the probability density of the random variable $\rvNextgen$ 
such that $\rvNextgen=\rvGen+\eta g(\rvGen)$ where  $0 < \eta < \min(1/b,h_0/a)$.
Then there exists a
positive constant $c$ such that for all $\epsilon>0$:
\tightDisplayBegin{0.5pt}
\[
\ourKL(\nextgenp) \leq \ourKL(\genp) 
- \eta \int p_*(\any) \; g(\any)^\top\nabla D(\any) \; d \any
+ c \eta^2  + c \eta \epsilon . 
\]
\tightDisplayEnd
\label{thm:grad}
\end{theorem}
The consequences of the theorem become clear when we choose 
$g(\any)= s(\any) \nabla D(\any)$ (where $s(\any)>0$ is an arbitrary scaling factor).  
By doing so and letting $\epsilon=\eta$, we have: 
\begin{equation}
\ourKL(\nextgenp) \leq \ourKL(\genp) 
- \eta \int p_*(\any)  s(\any) \; \|\nabla D(\any)\|_2^2  \; d \any
+ O(\eta^2) . 
\label{eq:onestep}
\end{equation} 
This means that by letting $g(\any)= s(\any) \nabla D(\any)$, 
the objective value $\ourKL(\cdot)$ will be reduced for a
sufficiently small $\eta$ 
unless %
$\int p_*(\any)  s(\any) \; \|\nabla D(\any)\|_2^2  \; d\any$ vanishes.  
The vanishing condition implies that $D(\any)$ is approximately a
constant when 
$p_*$ has full support on $\dataspace$. 
In this case, the discriminator is unable to differentiate
the real data from the generated data.
Thus, it is implied that  letting $g(\any)=s(\any) \nabla D(\any)$ makes 
the probability density of generated data closer to that of real data 
until the discriminator becomes unable to distinguish the real data and generated data. 

We note that taking $g(\any)= s(\any) \nabla
D(\any)$ is analogous to a gradient descent step of 
$\ourKL(p)$ in the
function space so that a step is taken to modify the function instead of the model parameters.  
Therefore, our theory leads to a
functional gradient view of variable transformation that {\em can always
improve the quality of the generator} --- when the quality is measured by the KL-divergence between
the true data and the generated data.   

If we repeat the process described above, 
Algorithm \ref{alg:fg} is obtained.  
We call it {\em composite functional gradient learning of
  GAN (CFG-GAN)}, or simply {\em CFG}.
\orgAlg\ forms %
$g_{t}$ by directly using the functional gradient $\nabla D_{t}(x)$, as suggested by 
our theory. 
If \eqref{eq:onestep} holds, and if we choose 
$\eta_t$$=$$\eta$$=$$\epsilon$$=$$O(1/\sqrt{T})$, then
by cascading \eqref{eq:onestep} from $t$$=$$1$ to $T$ 
we obtain the following bound: 
$\frac{1}{T}\sum_{t=1}^{T} \int p_*(x) s_t(x) \|\nabla D_{t}(x)\|^2 d x$
$=$
$\frac{1}{T} \left( O(T \eta) + \eta^{-1} \; \ourKL(p_0)\right)$ $=$ $O(T^{-1/2})$, 
where $p_0$ is the density of $G_0(z)$.
This means that 
as $t$ increases, $\nabla D_{t}(x) \to 0$ and thus $D_{t}(x)$ approaches 
a constant, assuming 
$p_*$ has full support on $\dataspace$. 
That is, in the limit, the discriminator is unable to differentiate the real
data from the generated data.

We note that \cite{LJT17} describes a related but different cascading process 
motivated by Langevin dynamics sampling.  
The Langevin theory requires repeated noise addition in the generation process. 
Our generation is simpler as there is no need for noise addition. 

\section{Composite Functional Gradient Algorithms}
Starting from the \orgAlg\ algorithm above, we empirically explored algorithms 
on image generation.  
In this section we parameterize $\justD$ and denote the model parameters by $\prm{\justD}$.  
\begin{figure}[t]
\centering
\begin{center}
\includegraphics[width=2.7in]{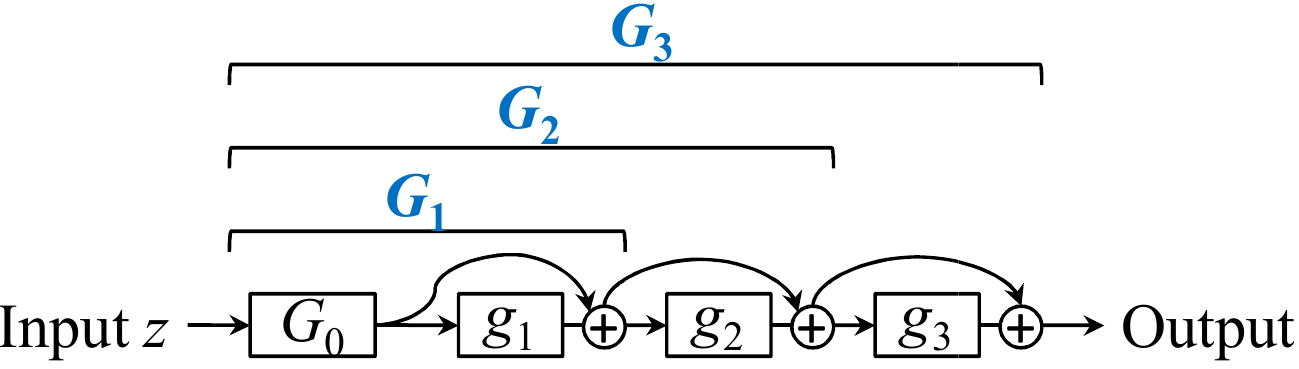} %
\vskip -0.05in
\caption{\label{fig:g} \capfont
  Generator network automatically derived by \orgAlg\ or \ourAlg\ 
  at $t=3$. 
  `$\oplus$' indicates addition. 
}
\end{center}
\end{figure}
\begin{algorithm}[t]
  \caption{\ourAlg: Incremental CFG}  
  \label{alg:ourAlg}
\begin{algorithmic}[1]
\REQUIRE  
  a set of training examples $\realset$, 
  prior $\prior$, 
  initial generator $\Gidx{0}$, 
  discriminator $\justD$. %
  Meta-parameters: $T$, 
                   mini-batch size $\minib$, 
                   discriminator update frequency $\dupd$. 
\FOR {$t=1,2,\ldots,T$}
  \FOR {$\dupd$ steps}
    \STATE Sample $\realidx{1},\ldots,\realidx{\minib}$ from $\realset$. 
    \STATE Sample $\inpt_1,\ldots,\inpt_{\minib}$ according to $\prior$.  
    \STATE Update $\justD$ by descending the stochastic gradient: 
    \begingroup
    \setlength{\medmuskip}{0mu}
    \begin{small}
    $\nabla_{\prm{\justD}} \frac{1}{\minib} \sum_{i=1}^\minib \left[ \ln(1+e^{-\justD(\realidx{i})}) + \ln(1+e^{\justD(\Gidx{t-1}(\inpt_i))}) \right]$

    \end{small}
    \endgroup
    \label{line:Dobj}   
  \ENDFOR
  \STATE $\gidx{t}(x) \leftarrow s_t(x) \nabla \justD(x)$ $~~~$ (e.g., $s_t(x)=1$)
  \STATE $\Gidx{t}(\inpt) \leftarrow  \Gidx{t-1}(\inpt) + 
                                \etaidx{t} \gidx{t}(\Gidx{t-1}(\inpt))$, for $\etaidx{t}>0$. \label{line:g2}
\ENDFOR
\RETURN generator $\Gidx{T}$ 
(and the updated $\justD$ if so desired). 
\end{algorithmic}
\end{algorithm}
\subsection{\ourAlg: Incremental CFG}

The first change made to \orgAlg\ is to make it incremental similar to GAN, 
resulting in {\em incremental CFG} (\ourAlg, Algorithm \ref{alg:ourAlg}). 
\orgAlg\ (Algorithm \ref{alg:fg}) optimizes the discriminator to convergence in every iteration. 
On image generation, 
this tends to cause the discriminator to overfit, 
which turned out to be much more harmful than underfitting 
as an overfit discriminator grossly violates the $\epsilon$-approximation condition 
(Section \ref{sec:assump}) 
thus grossly pushing the generator in a wrong direction.   
Also, updating the discriminator to convergence is computationally impractical 
apart from whether or not doing so is helpful/harmful.
\ourAlg\ incrementally updates a discriminator little by little interleaved with the generator updates, 
so that the generator can keep providing new and more challenging examples to 
prevent the discriminator from overfitting.  
Note that here we broadly use the term ``overfit'' for failure to generalize to unseen data 
(after fitting to observed data).  
This includes cases on unseen data in the {\em low-density region} 
according to our assumption, which is {\em outside of manifolds} according to 
the view of disjoint low-dim data manifolds 
of \cite{AB17}, 
and so 
our observation is in line with \cite{AB17} in that 
the shape of distributions could be problematic. 
However, we handle the issue differently.  
Instead of changing loss (leading to WGAN), 
we deal with it by {\em early stopping} of discriminator training 
(like GAN) 
and functional gradient learning in generator update
(unlike GAN).  
The latter ensures improvement of generator so that it keeps challenging
the discriminator, and we will later revisit this point. 

\ourAlg\ shares nice properties with \orgAlg. %
The generator is guaranteed to improve with each update 
to the extent that the assumptions hold; therefore, it is expected to be stable. 
There is no need to design a complex generator model. 
The generator model is automatically and implicitly 
derived from the discriminator
and %
{\em grows} as training proceeds (see Figure \ref{fig:g}).  
A shortcoming, however, is that the implicit generator network can become very large. 
At time $t$, %
computation of $\Gidx{t}(\inpt)$ starting from scratch is in $O(t)$; therefore, 
performing $T$ iterations of training could be in $O(\sum_{t=1}^T t) = O(T^2)$.
We found that image generation requires a 
large $T$ 
to the extent that it is computationally problematic,  
causing slow training and slow generation. 

\cite{NS18} recently proposed a related method for fine-tuning WGAN based on different motivations. 
We note that their method would also suffer from the same issue if used for image generation from scratch.

A partial remedy, which speeds up training (but not generation), is to have an {\em input pool} of a fixed size. 
That is, 
we restrict the input distribution to be on a finite set $\inptPl$ 
and for every input $\inpt \in \inptPl$, maintain $\Gidx{t'}(\inpt)$ for the latest $t'$ for which $\Gidx{t'}(\inpt)$ was computed.  
By doing so, 
when $\Gidx{t}(\inpt)$ needs to be computed, one can start from 
$\Gidx{t'}(\inpt)$ instead of starting over from $\Gidx{0}(\inpt)$, which saves computation time. 
However, this remedy solves only a part of the problem. 

\begin{algorithm}[t]
  \caption{\xourAlg: Approximate ICFG}  
  \label{alg:xourAlg}
\begin{algorithmic}[1]
\REQUIRE
a set of training examples $\realset$, prior $\prior$, 
approximator $\xG$ at its initial state, 
discriminator $\justD$. \\
Meta-parameters: input pool size, $(T,\minib,\dupd)$ for \ourAlg. 
\LOOP
  \STATE $\inptPl \leftarrow$ an input pool sampled according to $\prior$. \label{lno:pl-one}
  \STATE $\xprior \leftarrow$ the uniform distribution over $\inptPl$. \label{lno:pl-two} 
  \STATE $\justG$,$\justD \leftarrow$ output of \ourAlg\ using $\realset,\xprior,\xG,\justD$ as input. 
  \STATE {\bf if} exit criteria are met {\bf then} {\bf return} generator $\justG$ {\bf fi} \label{xourAlg:eval}
  \STATE Update $\xG$ to minimize $\sum_{\inpt \in \inptPl} \frac{1}{2}\|\xG(\inpt)-\justG(\inpt)\|^2$ \label{lno:xgupd}
\ENDLOOP
\end{algorithmic}
\end{algorithm}
\subsection{\xourAlg: Approximate incremental CFG} 

As a complete solution to the issue of large generators, we propose  
{\em Approximate \ourAlg} ({\em \xourAlg}, Algorithm \ref{alg:xourAlg}). 
\xourAlg\ periodically compresses the generator obtained by \ourAlg, by training 
an {\em approximator} of a fixed size that approximates the behavior of the generator obtained by \ourAlg.  
That is, 
given a definition of an approximator $\xG$ and its initial state, 
\xourAlg\ repeatedly alternates the following.  
\begin{itemize}
\setlength\itemsep{0em}
\item Using the approximator $\xG$ as the initial generator, %
perform $T$ iterations of \ourAlg\ to obtain generator $\justG$.  %
\item Update the approximator $\xG$ to achieve $\xG(\inpt) \approx \justG(\inpt)$. 
\end{itemize} 

The generator size is again in $O(T)$, but unlike \ourAlg, 
$T$ for \xourAlg\ can be small (e.g., $T=10$), 
thus \xourAlg\ is efficient. 
We use the idea of an input pool above for speeding up training, and instead of keeping the same pool to the end, we refresh the pool 
$\inptPl$ in every iteration of \xourAlg\ (Lines \ref{lno:pl-one}\&\ref{lno:pl-two} of Algorithm \ref{alg:xourAlg}).  
For speed, the values $\Gidx{t'}(\inpt)$ for $\inpt \in \inptPl$ for the latest $t'$ should be kept 
not only for use in \ourAlg\ but also  
for preparing the training data $\{ (\inpt,\justG(\inpt))~|~\inpt \in \inptPl \}$ 
for the training of the approximator $\xG$ (Line \ref{lno:xgupd}). 

A small pool size $\inptPlSz$ and a small $T$ would reduce the runtime of one iteration, 
but they would increase the number of required iterations, as they reduce the amount of 
the improvement achieved by one iteration of \xourAlg, and so a trade-off should be found empirically. 
In particular, 
approximation typically causes some degradation, and so it is important to set $T$ and $\inptPlSz$ to sufficiently large values 
so that the amount of the generator improvement exceeds the amount of degradation caused by approximation. 
In our experiments, however, tuning of meta-parameters turned out to be relatively easy; essentially
one set of meta-parameters achieved stable training in all the tested settings across datasets and 
network architectures, as described later. 

\begin{algorithm}[b]
\caption{GAN \cite{GAN14}}
  \label{alg:gan}
\begin{algorithmic}[1]
\REQUIRE $\realset$, $\prior$, discriminator $\ganD$, $\ganG$. 
         Meta-parameters $\minib,\dupd$.  
\REPEAT
  \FOR {$\dupd$ steps}
    \STATE Sample $\realidx{1},\ldots,\realidx{\minib}$ from $\realset$. 
    \STATE Sample $\inpt_1,\ldots,\inpt_\minib$ according to $\prior$.  
    \STATE Update $\ganD$ by ascending the stochastic gradient: \label{line:ganDobj} 
    \begin{center}
    $\nabla_{\prm{\ganD}} \frac{1}{\minib} \sum_{i=1}^{\minib} \left[ \ln \ganD(\realidx{i}) + \ln(1-\ganD(\ganG(\inpt_i))) \right]$ 
    \end{center}
  \ENDFOR
  \STATE Sample $\inpt_1,\ldots,\inpt_{\minib}$ according to $\prior$.  
  \STATE Update $\ganG$ by descending the stochastic gradient: 
    \begin{center}
    $\nabla_{\prm{\ganG}} \frac{1}{\minib} \sum_{i=1}^{\minib} \ln(1-\ganD(\ganG(\inpt_i)))$ \label{line:ganGupd}
    \end{center}
\UNTIL{exit criteria are met}
\RETURN generator $\ganG$ 
\end{algorithmic}
\end{algorithm}
\subsection{Relation to GAN}
\label{sec:gan}

We show that GAN (Algorithm \ref{alg:gan}) with 
the logistic model (as is typically done) 
is closely related to 
a special case of \xourAlg\ 
that uses an extreme setting. 
This viewpoint leads to a new insight into GAN's instability. 

We start with the fact that 
GAN with the logistic model (and so $\ganD(x)=\frac{1}{1+\exp(-\justD(x))}$) 
and \ourAlg\ share the discriminator update procedure as both 
minimize the logistic loss.  
This fact becomes more obvious when we plug $\ganD(x)=\frac{1}{1+\exp(-\justD(x))}$ into 
Line \ref{line:ganDobj} of Algorithm \ref{alg:gan}. 

Next, we show that 
{\em the generator update of GAN is equivalent to coarsely approximating a generator produced by \ourAlg\ with $T$=$1$}. 
First note that GAN's generator update (Line \ref{line:ganGupd} of Algorithm \ref{alg:gan}) 
requires the gradient $\nabla_{\prm{\ganG}} \ln(1-\ganD(\ganG(z)))$.  
Using $\ganD(x)=\frac{1}{1+\exp(-\justD(x))}$ again, and writing $[v]_i$ for the $i$-th component of vector $v$, 
the $k$-th component of this gradient can be written 
as: 
\newcommand{\newG}{G'}
\newcommand{\sgan}{s_0}
\newcommand{\slogd}{s_1}
\tightDisplayBegin{0.5pt}
\begingroup
\begin{align}
 & \textstyle
    \left[ \nabla_{\prm{\ganG}} \ln(1-\ganD(\ganG(\inpt))) \right]_k
  = \left[ \nabla_{\prm{\ganG}} \ln\frac{\exp(-\justD(\ganG(\inpt)))}{1+\exp(-\justD(\ganG(\inpt)))} \right]_k \nonumber \\[-0.05in]
 & ~~~~~~~~~~~~~
  \textstyle
  = -\sgan(\ganG(\inpt)) \sum_j \left[ \nabla \justD(\ganG(\inpt)) \right]_j \frac{\partial [\ganG(\inpt)]_j}{\partial [\prm{\ganG}]_k}
  \label{eqn:ganupd}
\end{align}
\endgroup
\tightDisplayEnd
where 
$
\sgan(\any)=\frac{1}{1+\exp(-\justD(\any))}
$, 
resulting from differentiating $f(y)=-\ln \frac{\exp(-y)}{1+\exp(-y)}$ at $y=\justD(\any)$. 
Now suppose that we apply \ourAlg\ with $T=1$ to a generator $\ganG$  
to obtain a new generator %
$
\newG(\inpt) = \ganG(\inpt) + \eta g(\ganG(\inpt)) = \ganG(\inpt) + \eta s(\ganG(\inpt)) \nabla \justD(\ganG(\inpt)),
$
%
and then 
we update $\ganG$ to approximate $\newG$ so that 
$
\sum_\inpt \frac{1}{2} \left\| \newG(\inpt) - \ganG(\inpt) \right\|^2
$
is minimized as in Line \ref{lno:xgupd} of \xourAlg.  
To take one step of gradient descent for this approximation, we need 
the gradient 
$\nabla_{\prm{\ganG}} \frac{1}{2} \left\| \newG(\inpt) - \ganG(\inpt) \right\|^2$, 
and its $k$-th component is
{\small
$
-\sum_j \left[ \newG(\inpt) - \ganG(\inpt) \right]_j \frac{\partial [\ganG(\inpt)]_j}{\partial [\prm{\ganG}]_k} 
= -\eta s(\ganG(\inpt)) \sum_j \left[ \nabla \justD(\ganG(\inpt)) \right]_j \frac{\partial [\ganG(\inpt)]_j}{\partial [\prm{\ganG}]_k}~.
$}
By setting the scaling factor $s(\any)=\sgan(\any)/\eta$, 
this is exactly the same as (\ref{eqn:ganupd}), required for the GAN generator update. 
(Recall that our theory and algorithms accommodate an arbitrary data-dependent scaling factor $s(\any)>0$.)

Thus, algorithmically GAN 
is closely related to 
a special case of \xourAlg\ 
that does the following: 
\begin{itemize}%
\setlength\itemsep{0em}
\item Set $T$=1
so that \ourAlg\ updates the generator {\em just once}. 
\item To update the approximator, take {\em only one} gradient descent step with {\em only one} mini-batch, 
      instead of optimizing to the convergence with many examples.   
      Therefore, the degree of approximation could be poor.  
\end{itemize}
The same argument applies also to the \logd-trick variant of GAN 
by replacing $\sgan(\any)=\frac{1}{1+\exp(-\justD(\any))}$ with 
$\slogd(\any)=\frac{1}{1+\exp(\justD(\any))}$. 
When 
generated data $\gen$ is very far from the real data and so
$\justD(\gen) \ll 0$, 
we have $\sgan(\gen) \approx 0$ (without the \logd\ trick), 
which would make the gradient (required for updating the GAN generator)
vanish, as noted in \cite{GAN14}, even though the generator is poor and so requires updating.  
In contrast, we have $\slogd(\gen) \approx 1$ (with the \logd\ trick)  
in this poor generator situation, 
which 
is more sensible as well as more similar to our choice ($s(\any)=1$) 
for the \xourAlg\ experiments. %

\mypara{Why is GAN unstable?}
In spite of their connection, 
GAN is unstable, and \xourAlg\ with appropriate meta-parameters is stable (shown later).  
Thus, we figure that GAN's instability derives from what is unique to GAN, the two bullets above
-- {\em an extremely small} $T$ and {\em coarse approximation}. 
Either can cause degradation of the generator, 
leading to instability.

We have contrasted GAN's generator update with \xourAlg's {\em approximator} update. 
Now we compare it with \ourAlg's {\em generator} update to consider 
the algorithmic merits of our functional gradient approach.  
The short-term goal of generator update 
can be regarded as the increase of 
the discriminator output on generated data, 
i.e., to have $D(G_{t+1}(\inpt)) > D(G_t(\inpt))$ for any $\inpt$$\sim$$\prior$.  
\ourAlg\ updates the generator by 
$G_{t+1}(\inpt)=G_t(\inpt)+\eta \nabla D(G_t(\inpt))$, 
and so with small $\eta$, 
$D(G(\inpt))$ is guaranteed %
to increase for any $\inpt$.  
This is because by definition $\nabla D(G_t(\inpt))$ is the direction that increases
the discriminator output for $\inpt$, 
and it is {\em precisely} obtained on the fly for {\em every} $\inpt$ 
at the time of generation. 
By contrast, GAN {\em stochastically} 
and {\em approximately}
updates $\prm{G}$ using a {\em small sample} 
(one mini-batch SGD step backpropagating $\nabla \justD$), 
and so GAN's update can be noisy, 
which can lead to instability through generator degradation.  
\section{Experiments}
\newcommand{\gano}{GAN0}
\newcommand{\gand}{GAN1}
\newcommand{\wgangp}{WGANgp}
\newcommand{\trdiff}{$|D(\mbox{real})$$-$$D(\mbox{gen})|$}
\newcommand{\trdiffs}{$\Delta_D$}
\newcommand{\setupsec}{\mypara}

We tested \xourAlg\ on the image generation task. 

\subsection{Experimental setup}
\label{sec:setup}
\setupsec{Baseline methods} 
For comparison, we also tested the following three methods: 
the original GAN {\em without} the \logd\ trick ({\em \gano} in short), 
GAN {\em with} the \logd\ trick ({\em \gand}), and 
WGAN with the gradient penalty ({\em \wgangp}) \cite{WGANgp17}.  
The choice of \gano\ and \gand\ is 
due to its relation to \xourAlg\ as analyzed above.  
\wgangp\ was chosen as a representative of state-of-the-art methods, 
as it was shown to rival or outperform a number of previous methods such as 
the original WGAN with weight clipping \cite{WGAN17}, 
Least Squares GAN \cite{LSGAN17}, 
Boundary Equilibrium GAN \cite{BEGAN17}, 
GAN with denoising feature matching \cite{WB17}, 
and
Fisher GAN \cite{WGANfs17}.

%
\newcommand{\pr}{p}
\setupsec{Evaluation metrics}
Making reliable likelihood estimates with generative adversarial models 
is known to be challenging \cite{TOB16}, and 
we instead focused on evaluating the visual quality of generated images, 
using datasets that come with labels for classification.  
We measured the {\em inception score} 
\cite{SGZCRC16}.  
The intuition behind this score is that high-quality images should lead to high confidence in classification. 
It is defined as 
$\exp({\mathbb E}_\gen{\rm KL}(\pr(y|\gen)||\pr(y)))$ 
where $\pr(y|\gen)$ is the label distribution conditioned on generated data $\gen$ 
and $\pr(y)$ is the label distribution over the generated data.  
Following previous work, e.g., \cite{YKBP17,CLJBL17}, 
the probabilities were estimated by a classifier trained with the labels provided with the datasets  
(instead of the ImageNet-trained {\em inception} model used in \cite{SGZCRC16}) 
so that the image classes of interest were well represented in the classifier.  
We, however, call this score the `inception score', %
following custom.
We also compared the label distributions over generated data and real data, but 
we found that in our settings this measure roughly correlates to the inception score
(generally, a very good match when the methods produce decent inception scores), 
and so we do not report it to avoid redundancy.
We note that these metrics are limited, e.g., they would not detect mode collapse
or missing modes 
within a class.  Apart from that, 
we found the inception score to generally correspond to human perception well.

\setupsec{Data}
We used  
MNIST, 
the Street View House Numbers dataset (SVHN) \cite{SVHN},  and 
the large-scale scene understanding (LSUN) dataset. 
These datasets are provided with 
class labels (digits `0' -- `9' for MNIST and SVHN and 10 scene types for LSUN). 
A number of studies have used only one LSUN class (`bedroom').
Since a single-class dataset would preclude 
evaluation using class labels, %
we instead generated a balanced two-class dataset 
using the same number of 
images 
from the `bedroom' class and the `living room' class (LSUN BR+LR).  
Similarly, we generated a balanced dataset from 
`tower' and `bridge' 
(LSUN T+B). 
The number of 
real images used for training was 
60K (MNIST), 521K (SVHN), %
2.6 million (LSUN BR+LR), 
and 1.4 million (LSUN T+B). %
The LSUN images 
were shrunk and cropped into 64$\times$64 %
as in previous studies. 
The pixel values were scaled into $[-1,1]$.  

\setupsec{Network architectures} 
\label{sec:net}
The tested methods require 
as input a network architecture of a discriminator and that of an approximator or a generator.   
Among the 
numerous network architectures we could experiment with, we focused on two types 
with two distinct merits -- good results and simplicity. 

The first type (convolutional; stronger) aims at complexity appropriate for the dataset 
so that good results can be obtained.  
On MNIST and SVHN, we used an extension of DCGAN \cite{DCGAN15}, 
adding 1$\times$1 convolution layers.  
Larger (64$\times$64) images of LSUN were found to benefit from more complex networks, and so 
we used a residual net (ResNet) \cite{HZRS15}
of four residual blocks, 
which is a simplification from the \wgangp\ code release, 
for both the discriminator and the approximator/generator. 
Details are given in the Appendix.  

\newcommand{\BNshort}{batch normalization}
These networks include batch normalization 
layers \cite{IS15}. 
The original study %
states that 
\wgangp\ does not work well with a discriminator with \BNshort. 
Although it would be ideal to use exactly the same 
networks 
for all the methods, 
it would be rather unfair for the other methods if we always remove \BNshort.  
Therefore, %
in each setting, 
we tested \wgangp\ 
with the options of either removing \BNshort\ only from $D$ or from both $D$ and $G$, 
and picked the best. 
(We also tried other normalizations such as layer normalization but did not see any merit.) 
In addition, we tested some cases 
without \BNshort\ anywhere for all the methods. 

The second type (fully-connected $\xG$ or $\ganG$; weaker) 
uses a minimally simple approximator/generator, 
consisting of two 512-dim fully-connected layers 
with ReLU, followed by the output layer with $\tanh$, 
which has a merit of simplicity, requiring less design effort. 
We combined it with a convolutional discriminator, the DCGAN extension above. 

\setupsec{\xourAlg\ implementation details}
To speed up training, we limited the number of epochs of 
the approximator training in \xourAlg\ 
to 10 while reducing the learning rate by multiplying by 0.1 whenever 
the training loss stops going down.  
The scaling function $s(x)$ in \ourAlg\ 
was set to $s(x)=1$.  
\newcommand{\iniG}{G_{\rm rand}}
To initialize the approximator $\xG$ for \xourAlg, 
we first created a simple generator $\iniG(\inpt)$ consisting of a projection layer 
with random weights 
(Gaussian with 0 mean and 0.01 stddev)
to produce 
the desired dimensionality, 
and then trained $\xG$ to approximate 
$\iniG$. 
The training time reported below includes the time spent for this initialization. 

\begin{table}
\begin{small}
\begin{center} \begin{tabular}{|c|l|c|}
\hline
$\minib$ & mini-batch size & 64  \\
\hline
$|\inptPl|$ & input pool set size & 10$\minib$ \\
\hline
$\dupd$ & discriminator update frequency & 1 \\
\hline
$T$ & number of iterations in \ourAlg\ & 25\\
\hline
\end{tabular} \end{center}
\end{small}
\vskip -0.15in
\caption{ \capfont \label{tab:meta}
Meta-parameters for \xourAlg. 
}
\end{table}
\setupsec{Other details}
\label{sec:meta}
In all cases, the prior $\prior$ was set to generate 
100-dimensional 
Gaussian 
vectors with zero mean and standard deviation 1. 
All the experiments were done using a single NVIDIA Tesla P100. %

The meta-parameter values for \xourAlg\ were fixed to those in Table \ref{tab:meta}
except when we pursued smaller values for $T$ for practical advantages 
(described below).  
For GAN, we used the same mini-batch size as \xourAlg, 
and we set the discriminator update 
frequency $\dupd$ to 1 as other values led to poorer results. 
The SGD update was done 
with {\em rmsprop} \cite{rmsprop} for \xourAlg\ and GAN.  %
The learning rate for rmsprop was fixed for \xourAlg, %
but we tried several values for GAN as it turned out to be critical.   
Similarly, for \xourAlg, 
we found it important to set the step size $\eta$ 
for the generator update in \ourAlg\ to an appropriate value. 
The SGD update for \wgangp\ was done with Adam \cite{Adam15}
as in the original study. %
We set the meta-parameters for \wgangp\ to the suggested values, 
except that we tried several values for the learning rate.  %
Thus, the amount of tuning effort was about the same for all 
but \wgangp, which required additional search for the normalization options.  
Tuning was done based on 
the inception score on the validation set of 10K input vectors 
(i.e., 10K 100-dim Gaussian vectors), and we report inception scores on the 
test set of 10K input vectors, disjoint from the validation set. 

\begin{figure*}[h]
\centering
\includegraphics[width=6.2in]{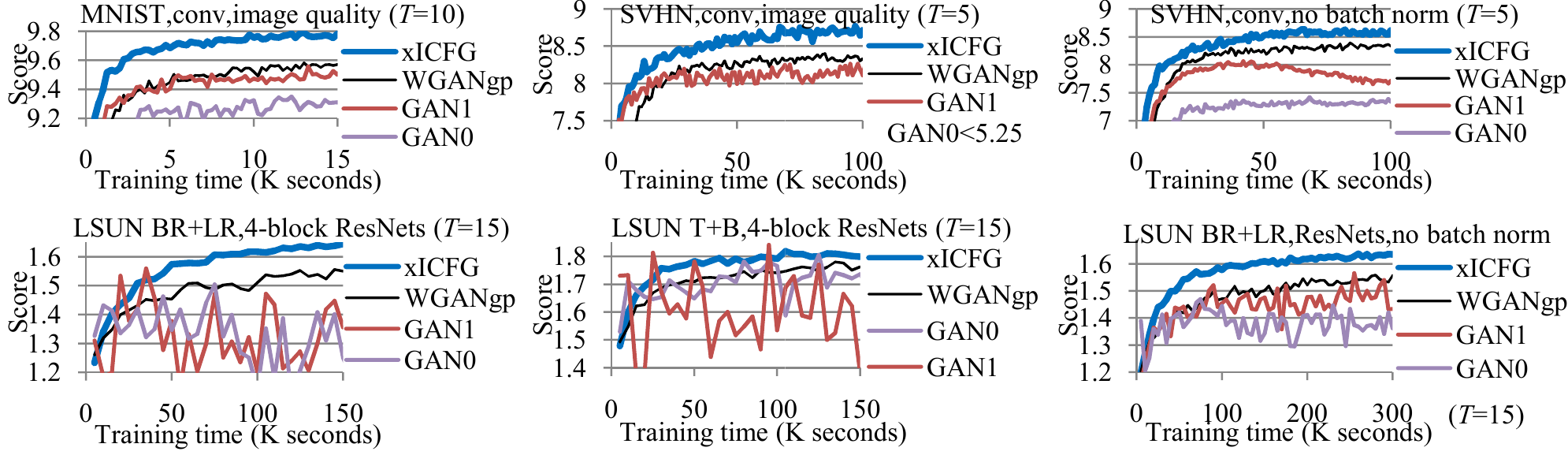}
\vskip -0.2in
\caption{\label{fig:conv-i} \capfont
Image quality (measured by the inception score) in relation to training time.  Convolutional networks. 
The legends are sorted from the best to the worst. 
The two graphs on the right are without batch normalization anywhere for all the methods. 
}
\end{figure*}
\begin{figure*}[h]
\centering
\includegraphics[width=6.25in]{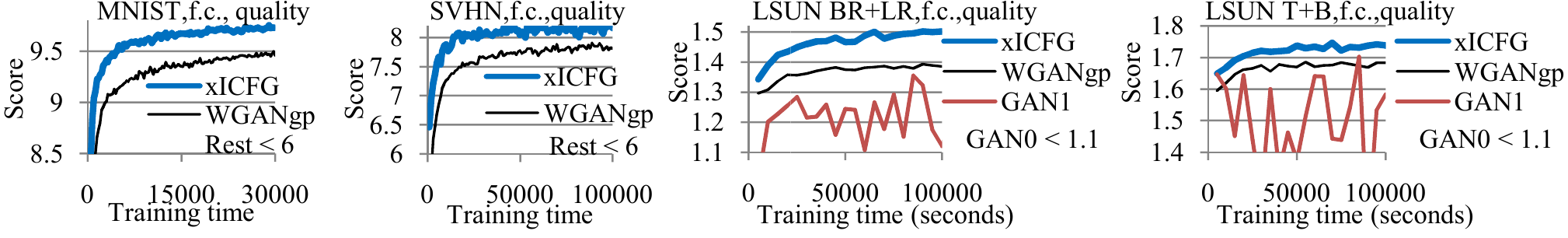}
\vskip -0.2in
\caption{\label{fig:fc-i} \capfont
Image quality (measured by the inception score) in relation to training time.  Fully-connected approximators/generators. 
}
\end{figure*}
\subsection{Results}

\newcommand{\btw}{0.2in}
\ifdefined \nomnist
\else
\newcommand{\mwd}{1in}
\newcommand{\mrwd}{6in}
\newcommand{\mln}{0.33}

\newcommand{\mcii}{15}
\newcommand{\mciik}{15k}

\newcommand{\mfii}{30}
\newcommand{\mfiik}{30k}
\newcommand{\msiz}{8x2}
\newcommand{\msiztmp}{10x3}
\newcommand{\mrsiz}{30x1}
\begin{figure}[h]
\captionsetup[subfigure]{labelformat=empty}
\centering
\begin{subfigure}[b]{\mln\linewidth}
\begin{center}
\includegraphics[width=\mwd]{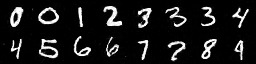}
\vskip -0.05in
\caption{\label{fig:m-xcfg-c-img} 
\xourAlg\ 
(9.78) %
}
\end{center}
\end{subfigure}%
\begin{subfigure}[b]{\mln\linewidth}
\begin{center}
\includegraphics[width=\mwd]{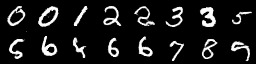}
\vskip -0.05in
\caption{\label{fig:m-best-c-img} 
Best baseline
(9.57) %
}
\end{center}
\end{subfigure}%
\begin{subfigure}[b]{\mln\linewidth}
\begin{center}
\includegraphics[width=\mwd]{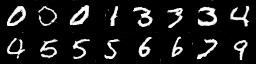}
\vskip -0.05in
\caption{\label{fig:m-worst-c-img} 
Worst baseline
(9.31) %
}
\end{center}
\end{subfigure}%
\vskip -0.2in
\caption{\label{fig:m-c-img} 
MNIST.  Using convolutional networks as in Fig \ref{fig:conv-i}. 
In this and all the image figures below, 
the images were randomly chosen and sorted by the predicted classes; 
the numbers are the inception scores averaged over 10K images.  
}
\vskip 0.05in
\captionsetup[subfigure]{labelformat=empty}
\begin{subfigure}[b]{\mln\linewidth}
\begin{center}
\includegraphics[width=\mwd]{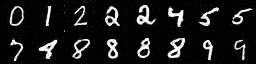}
\vskip -0.05in
\caption{\label{fig:m-xcfg-f-img} 
\xourAlg\ 
(9.72)
}
\end{center}
\end{subfigure}%
\begin{subfigure}[b]{\mln\linewidth}
\begin{center}
\includegraphics[width=\mwd]{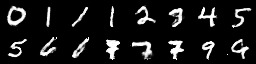}
\vskip -0.05in
\caption{\label{fig:m-best-f-img} 
Best baseline
(9.45) %
}
\end{center}
\end{subfigure}%
\begin{subfigure}[b]{\mln\linewidth}
\begin{center}
\includegraphics[width=\mwd]{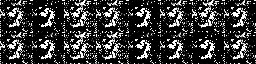}
\vskip -0.05in
\caption{\label{fig:m-worst-f-img} 
Worst baseline
(1.08)
}
\end{center}
\end{subfigure}%
\vskip -0.2in
\caption{\label{fig:m-f-img} \capfont 
MNIST.  Using fully-connected $\xG$ or $\ganG$ as in Fig \ref{fig:fc-i}.     
}
\vskip 0.05in
\fi
\newcommand{\srsz}{24x2}
\newcommand{\sgsz}{5x2}
\newcommand{\sgsztmp}{8x5}
\newcommand{\swd}{1.05in}
\newcommand{\srwd}{3.15in}
\newcommand{\sln}{0.33333}

\newcommand{\sii}{100}
\newcommand{\siik}{100k}
\ifdefined \nomnist
\begin{figure}
\captionsetup[subfigure]{labelformat=empty}
\fi
\centering
\begin{subfigure}[b]{\sln\linewidth}
\captionsetup{justification=centering}
\begin{center}
\includegraphics[width=\swd]{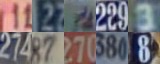}
\vskip -0.05in
\caption{\label{fig:s-xcfg-c-img} 
\xourAlg\
(8.72)
}
\end{center}
\end{subfigure}%
\begin{subfigure}[b]{\sln\linewidth}
\captionsetup{justification=centering}
\begin{center}
\includegraphics[width=\swd]{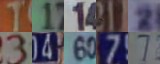}
\vskip -0.05in
\caption{\label{fig:s-best-c-img} 
Best baseline
(8.33) %
}
\end{center}
\end{subfigure}%
\begin{subfigure}[b]{\sln\linewidth}
\captionsetup{justification=centering}
\begin{center}
\includegraphics[width=\swd]{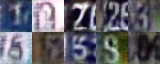}
\vskip -0.05in
\caption{\label{fig:s-worst-c-img} 
Worst baseline
(4.59)
}
\end{center}
\end{subfigure}%
\vskip -0.2in
\caption{\label{fig:s-c-img} \capfont
SVHN.  %
Using convolutional networks
as in Fig \ref{fig:conv-i}. 
\ifdefined \nomnist
In this and all the image figures below, 
the images were randomly chosen and sorted by the predicted classes; 
the numbers are the inception scores averaged over 10K images.  
\fi
}
\vskip 0.05in
\begin{subfigure}[b]{\sln\linewidth}
\captionsetup{justification=centering}
\begin{center}
\includegraphics[width=\swd]{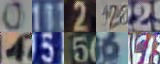}
\vskip -0.05in
\caption{\label{fig:s-xcfg-f-img} 
\xourAlg\
(8.16) %
}
\end{center}
\end{subfigure}%
\begin{subfigure}[b]{\sln\linewidth}
\captionsetup{justification=centering}
\begin{center}
\includegraphics[width=\swd]{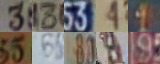}
\vskip -0.05in
\caption{\label{fig:s-best-f-img} 
Best baseline
(7.81)
}
\end{center}
\end{subfigure}%
\begin{subfigure}[b]{\sln\linewidth}
\captionsetup{justification=centering}
\begin{center}
\includegraphics[width=\swd]{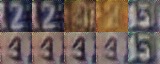}
\vskip -0.05in
\caption{\label{fig:s-worst-f-img} 
Worst baseline
(3.37)
}
\end{center}
\end{subfigure}%
\vskip -0.2in
\caption{\label{fig:s-img} \capfont
SVHN.  
Using fully-connected 
$\xG$ or $\ganG$ 
as in Fig \ref{fig:fc-i}. 
}
\end{figure}

\newcommand{\brwd}{6.36in}
\newcommand{\brsz}{12x2}

\newcommand{\bwd}{2.2in}
\newcommand{\bgsz}{5x1}
\newcommand{\bgsztmp}{4x6}
\newcommand{\bgszf}{5x1}
\newcommand{\bgszftmp}{4x5}

\newcommand{\bln}{0.3333333333333}
\newcommand{\bcii}{150}
\newcommand{\bfii}{100}
\newcommand{\bciik}{\bcii k}
\newcommand{\bfiik}{100k}
\begin{figure*}
\centering
\begin{subfigure}[b]{\bln\linewidth}
\begin{center}
\includegraphics[width=\bwd]{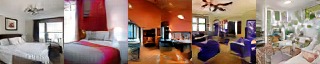}
\vskip -0.075in
\caption{\label{fig:b-xcfg-c-img} 
ResNets. \xourAlg\ (1.64)
}
\end{center}
\end{subfigure}%
\begin{subfigure}[b]{\bln\linewidth}
\begin{center}
\includegraphics[width=\bwd]{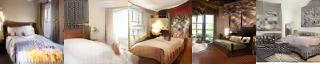}
\vskip -0.075in
\caption{\label{fig:b-best-c-img} 
ResNets. Best baseline (1.56)
}
\end{center}
\end{subfigure}%
\begin{subfigure}[b]{\bln\linewidth}
\begin{center}
\includegraphics[width=\bwd]{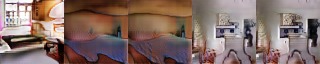}
\vskip -0.075in
\caption{\label{fig:b-bestgan-c-img} 
ResNets. Second best baseline (1.35) 
}
\end{center}
\end{subfigure}%
\quad
\vskip 0.025in
\begin{subfigure}[b]{\bln\linewidth}
\begin{center}
\includegraphics[width=\bwd]{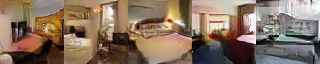}
\vskip -0.075in
\caption{\label{fig:b-xcfg-f-img} 
f.c. \xourAlg\ (1.50)
}
\end{center}
\end{subfigure}%
\begin{subfigure}[b]{\bln\linewidth}
\begin{center}
\includegraphics[width=\bwd]{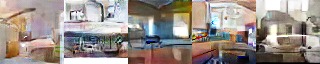}
\vskip -0.075in
\caption{\label{fig:b-best-f-img} 
f.c. Best baseline (1.39)
}
\end{center}
\end{subfigure}%
\begin{subfigure}[b]{\bln\linewidth}
\begin{center}
\includegraphics[width=\bwd]{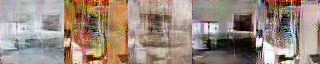}
\vskip -0.075in
\caption{\label{fig:b-bestgan-f-img} 
f.c. Second best baseline (1.12)
}
\end{center}
\end{subfigure}%
\vskip -0.2in
\caption{\label{fig:b-img} \capfont
LSUN bedrooms and living rooms (64$\times$64). 
(a-c) 4-block ResNets as in Fig \ref{fig:conv-i}. 
(d-f) Fully-connected $\xG$/$\ganG$ as in Fig \ref{fig:fc-i}. 
}
\end{figure*}

\newcommand{\twd}{2.2in}
\newcommand{\tgsz}{5x1}
\newcommand{\tgsztmp}{4x6}
\newcommand{\tgszf}{5x1}
\newcommand{\tgszftmp}{4x5}

\newcommand{\trsz}{12x2}
\newcommand{\trwd}{6.36in}

\newcommand{\tln}{0.3333333333333}
\newcommand{\tcii}{150}
\newcommand{\tfii}{100}
\newcommand{\tciik}{\tcii k}
\newcommand{\tfiik}{100k}
\begin{figure*}
\centering
\begin{subfigure}[b]{\tln\linewidth}
\begin{center}
\includegraphics[width=\twd]{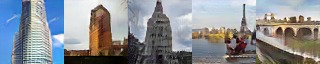}
\vskip -0.075in
\caption{\label{fig:t-xcfg-c-img} 
conv. \xourAlg\ (1.80) 
}
\end{center}
\end{subfigure}%
\begin{subfigure}[b]{\tln\linewidth}
\begin{center}
\includegraphics[width=\twd]{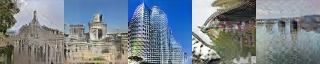}
\vskip -0.075in
\caption{\label{fig:t-best-c-img} 
conv. Best baseline (1.76)
}
\end{center}
\end{subfigure}%
\begin{subfigure}[b]{\tln\linewidth}
\begin{center}
\includegraphics[width=\twd]{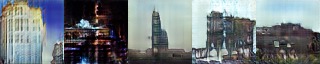}
\vskip -0.075in
\caption{\label{fig:t-bestgan-c-img} 
conv. Second best baseline (1.74) 
}
\end{center}
\end{subfigure}%
\quad
\vskip 0.025in
\begin{subfigure}[b]{\tln\linewidth}
\begin{center}
\includegraphics[width=\twd]{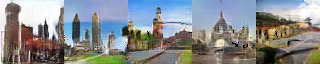}
\vskip -0.075in
\caption{\label{fig:t-xcfg-f-img} 
f.c. \xourAlg\ (1.74)
}
\end{center}
\end{subfigure}%
\begin{subfigure}[b]{\tln\linewidth}
\begin{center}
\includegraphics[width=\twd]{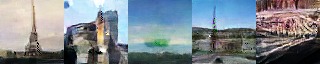}
\vskip -0.075in
\caption{\label{fig:t-best-f-img} 
f.c. Best baseline (1.68)
}
\end{center}
\end{subfigure}%
\begin{subfigure}[b]{\tln\linewidth}
\begin{center}
\includegraphics[width=\twd]{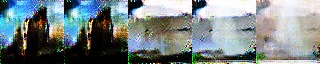}
\vskip -0.075in
\caption{\label{fig:t-bestgan-f-img} 
f.c. Second best baseline (1.58)
}
\end{center}
\end{subfigure}%
\vskip -0.2in
\caption{\label{fig:t-img} \capfont
LSUN towers and bridges (64$\times$64). 
(a-c) 4-block ResNets as in Fig \ref{fig:conv-i}. 
(d-f) Fully-connected $\xG$/$\ganG$ as in Fig \ref{fig:fc-i}.  
}
\end{figure*}

\newcommand{\ggbwdreal}{0.88in}
\newcommand{\ggbwdgood}{0.88in}
\newcommand{\ggbwdvar}{1.32in}
\newcommand{\ggblnreal}{0.28}
\newcommand{\ggblngood}{0.28}
\newcommand{\ggblnvar}{0.42}

\begin{figure}
\centering
\begin{subfigure}[b]{\ggblnreal\linewidth}
\begin{center}
\includegraphics[width=\ggbwdreal]{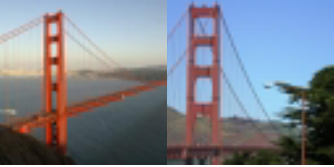}
\vskip -0.08in
\caption{\label{fig:goldreal} 
Real images. 
}
\end{center}
\end{subfigure}%
\begin{subfigure}[b]{\ggblngood\linewidth}
\begin{center}
\includegraphics[width=\ggbwdgood]{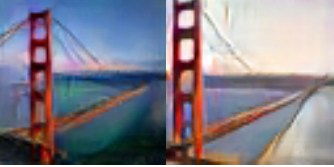}
\vskip -0.08in
\caption{\label{fig:goldgood} 
``Realistic'' 
}
\end{center}
\end{subfigure}%
\begin{subfigure}[b]{\ggblnvar\linewidth}
\begin{center}
\includegraphics[width=\ggbwdvar]{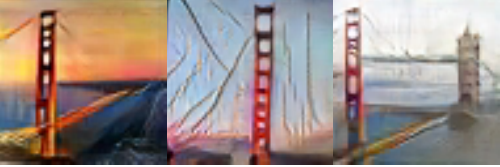}
\vskip -0.08in
\caption{\label{fig:goldvar} 
``Creative''
}
\end{center}
\end{subfigure}%
\vskip -0.2in
\caption{\capfont \label{fig:gold}
(a) Real Golden Gate Bridge images in LSUN T+B; the red tower has 4 grids.  
(b) Images generated by \xourAlg\ that look like Golden Gate Bridge though not perfect. 
(c) Images generated by \xourAlg\ that look like {\em modifications} of Golden Gate Bridge 
with more grids or connected with an object that is not there in reality.   
} 
\end{figure}

First, we report the inception score results.  
The scores of the real data (a held-out set of 10K images) 
are 9.91 (MNIST), 9.13 (SVHN), 1.84 (LSUN BR+LR), and 1.90 (LSUN T+B), respectively,
which roughly set the upper bounds 
that can be achieved by generated images.
Figure \ref{fig:conv-i} shows the score of generated images (in relation to training time) 
with the convolutional networks, %
including the two cases without batch normalization anywhere for all the methods 
(upper-right and lower-right).  
Recall that 
a smaller $T$ has practical advantages of a smaller generator resulting in faster generation and smaller footprints
while a larger $T$ stabilizes \xourAlg\ training by ensuring that training makes progress 
by overcoming the degradation caused by approximation.  
With convolutional approximators, 
we explored values for $T$ by a decrement of 5 
starting from $T$=25 (which works well for all) 
and found that $T$ can be reduced to 5 (SVHN), 10 (MNIST), and 15 (both LSUN) 
without negative consequences.  
The results in Figure \ref{fig:conv-i} were obtained by using these smaller $T$. 
\xourAlg\ generally outperforms the others. %
Although on LSUN datasets \gand\ occasionally exceeds
\xourAlg, inspection of generated images reveals that it suffers from 
severe mode collapse.
The 
results with the 
simple but weak 
fully-connected approximator/generator are shown in Figure \ref{fig:fc-i}.  
Among the baseline methods, only \wgangp\ succeeded in this setting, but 
its score fell behind \xourAlg.  
These results show that \xourAlg\ is effective and efficient.  

Examples of generated images are shown in Figures \ref{fig:m-c-img}--\ref{fig:t-img}.  
Note however that a small set of images %
may not represent the entire population well
due to variability.  
Looking through larger sets of generated images, 
we found that roughly, when the inception score is higher, the images are sharper 
and/or there are fewer images that are harder to tell what they are, and that 
when the score fluctuates violently (as \gand\ does on LSUN), severe mode collapse is observed. 
Overall, 
we feel that the images generated by \xourAlg\ are better than or 
at least as good as 
those of the best-performing baseline, \wgangp, 
one of the state-of-the-art methods. 

\begin{figure}[H]
\centering
\begin{subfigure}[b]{0.5\linewidth}
\begin{center}
\includegraphics[width=1.5in]{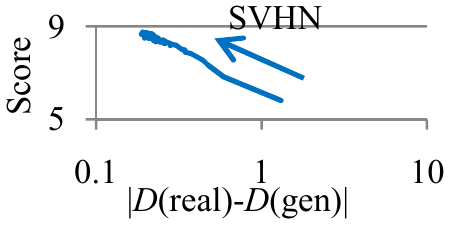}
\vskip -0.1in
\end{center}
\end{subfigure}%
\begin{subfigure}[b]{0.5\linewidth}
\begin{center}
\includegraphics[width=1.5in]{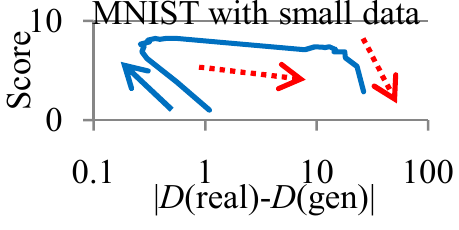}
\vskip -0.1in
\end{center}
\end{subfigure}%
\vskip -0.15in
\caption{\label{fig:cfg-trdiff} \capfont
Image quality ($y$-axis) vs. \trdiffs (=\trdiff, $x$-axis).  
\xourAlg. 
The arrows indicate the direction of time flow. 
A correlation is observed both when training is succeeding 
(blue solid arrows) and failing (red dotted arrows).       
}
\end{figure}
\mypara{Discriminator output values}

Successful training should make it harder and harder for the discriminator  
to distinguish real images and generated images, 
which would manifest as the discriminator output values for real images and generated images 
becoming closer and closer. 
In Figure \ref{fig:cfg-trdiff}, 
the $y$-axis is the inception score, and the $x$-axis is 
`\trdiff' (\trdiffs\ in short), which is the 
difference between the discriminator output values for real images and 
generated images averaged over time intervals of a fixed length, 
obtained as a by-product of the forward propagation for updating the discriminator. 
The arrows indicate the direction of time flow.  
When training is going well (indicated by blue solid arrows), \trdiffs\ decreases and the inception 
score improves as training proceeds. 
When it is failing, \trdiffs\ goes up rapidly and the inception scores degrades rapidly
 (red dotted arrows in Fig \ref{fig:cfg-trdiff} right).  
Here, the discriminator is overfitting to the small set of real data (1000 MNIST examples), 
violating the $\epsilon$-approximation assumption.  
That slows down and eventually stops the progress of the generator, resulting in 
the increase of \trdiffs.  In practice, training should be stopped before the rapid growth of \trdiffs.  
Thus, the decrease/increase of \trdiffs\ values (which can be obtained at almost no cost 
during training) can be used as an indicator of the status of \xourAlg\ training, 
similar to WGAN and in contrast to GAN.  
\mypara{Use of the approximator as a generator}
As noted above, a smaller generator resulting from a smaller $T$ has 
practical advantages, %
but a larger $T$ stabilizes training.  
One way to reduce generator size without reducing $T$ is to use the final approximator $\xG$ 
(after completing regular \xourAlg\ training)
as a generator, 
at the expense of performance degradation.  
We show below the inception scores of the final approximator 
($G_0$ in the final call of \ourAlg) 
in comparison with the final generator 
($G_T$ in the final call of \ourAlg) 
in the settings of Figure \ref{fig:conv-i} (convolutional). %
\begingroup
\begin{center}
\includegraphics[width=3in]{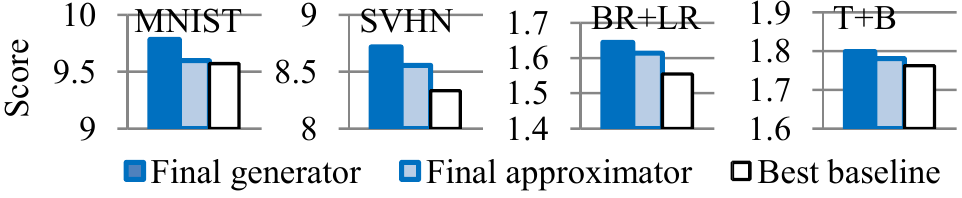}
\end{center}
\vskip -0.1in
Although the final approximator
underperforms the final generator 
as expected, 
it rivals and sometimes exceeds 
the best baseline method, whose 
generator has the same size 
as the approximator. 
Thus, use of the final approximator may be a viable option. 
The results also indicate that the 
good 
performance 
of \xourAlg\ 
is due to 
not only having a larger (and therefore more complex) 
generator but also stable and efficient training, which makes it possible to 
refine the approximator network to the degree that it can outperform the baseline methods.  

\mypara{Not memorization}
Finally, in Fig \ref{fig:gold} we show examples indicating 
\xourAlg\ does something more than memorization.

\section{Conclusion}
In the generative adversarial learning setting, 
we considered a generator that can be obtained using composite functional gradient learning.
Our theoretical results led to the new stable algorithm \xourAlg.  
The experimental results showed that \xourAlg\ generated equally good or better images 
than GAN and WGAN variants in a stable manner. %

\clearpage
\bibliography{cfg-icmlcr}
\bibliographystyle{icml2018}

\ifx \myarxiv \undefined 
\else
\onecolumn

\begin{center}
{\Large \bf
Appendix
}
\end{center}

\newcommand{\refsecassump}{\ref{sec:assump}} %
\newcommand{\refthmgrad}{\ref{thm:grad}}   %
\newcommand{\refsecsetup}{\ref{sec:setup}} %
\newcommand{\reffigconv}{\ref{fig:conv-i}} %
\newcommand{\reftabmeta}{\ref{tab:meta}} %
\appendix
\section{Main theorem and its proof}

\newcommand{\BlackBox}{\rule{1.5ex}{1.5ex}}  
\newenvironment{proof}{\par\noindent{\bf Proof\ }}{\hfill\BlackBox\\[2mm]}

\newcommand{\idealDb}{{\mathcal{D}_\beta}}

\newcommand{\Db}{D_\beta}

Theorem \ref{thm:bgrad} below, our main theorem, analyzes the extended KL-divergence for some $\beta \in (0.5,1]$ 
defined as follows: 
\[
L_\beta(p) := \int (\beta p_*(x) + (1-\beta) p(x)) \ln \frac{\beta
  p_*(x) + (1-\beta) p(x)}{(1-\beta) p_*(x) + \beta
  p(x)} d x~.  %
\]
Theorem \refthmgrad\ above is a simplified version of Theorem \ref{thm:bgrad}, 
and it can be obtained by setting $\beta=1$ in Theorem \ref{thm:bgrad}.  
The assumption of smooth and light-tailed $p_*$ in Section \ref{sec:assump}  
is precisely stated in Assumption \ref{asump:p} below.  

We first state the assumptions and then present Theorem \ref{thm:bgrad} and its proof.  

\subsection{Assumptions of Theorem \ref{thm:bgrad}}
\label{app:assump}

\paragraph{A strong discriminator} 
Given a set $\realset$ of real data, a set $\genset$ of generated data, 
and $\beta \in (0.5,1]$ indicating that the probability 
$P(\real\mbox{ is real})=\beta$ for $\real \in \realset$, and 
$P(\gen\mbox{ is real})=1-\beta$ for $\gen \in \genset$, 
assume that we can obtain a strong discriminator $\Db$ that solves 
the following weighed logistic regression problem: 
\newcommand{\bb}{\beta}
\begin{align*}
\Db \approx
  \arg\min_{\localD} \sum_{i: x_i \in \realset \cup \genset } 
    w_i \left[ \bb_i \ln (1 + \exp(-\localD(x_i)) ) + (1-\bb_i) \ln (1 + \exp( \localD(x_i)) ) \right] \\                   
    (\bb_i,w_i) = \left \{ \begin{array}{lll}
                             (\beta,   ~1/|\realset|)& \mbox{ for } x_i \in \realset & \mbox{ (real data) }\\
                             (1-\beta, ~1/|\genset|) & \mbox{ for } x_i \in \genset  & \mbox{ (generated data)}
                   \end{array} \right .                   
\end{align*}   
Define a quantity $\idealDb(\any)$ by 
\[
\idealDb(\any) :=\ln \frac{ \beta p_*(\any) + (1-\beta)\genp(\any)}{ (1-\beta)p_*(\any) + \beta\genp(\any) }
\] 
where $p_*$ and $p$ are the probability density functions of real data and generated data, 
respectively, 
and assume that there exists a positive number $B < \infty$ such that $|\idealDb(\any)| \le B$.  
Note that 
  if we choose $\beta <1$, then we can take $B =\ln (\beta/(1-\beta)) < \infty$, 
  and the assumption on $p_*$ and $p$ can be relaxed from being both nonzero (with $\beta=1$) 
  to not being zero at the same time. 
  However, in practice, one can simply take $\beta=1$, and that is what was done 
  in our experiments, because 
  the practical behavior of choosing $\beta \approx 1$ is similar to $\beta=1$.  

When the number of given examples is sufficiently large,  
the standard statistical consistency theory of logistic regression implies that 
\[
\Db(\any) \approx \idealDb(\any)~.
\]
Therefore, 
assume that 
the following $\epsilon$-approximation condition is satisfied
for a small $\epsilon>0$: 
\begin{equation}
\int p_*(\any) \max(1,\|\nabla \ln p_*(\any)\|) \left(\left|\Db(\any)-\idealDb(\any)\right|+\left|e^{\Db(\any)}-e^{\idealDb(\any)}\right|\right) d \any \le \epsilon~. \label{eq:eps}
\end{equation} 

\paragraph{Smooth light-tailed $p_*$} 
For convenience, we impose the following assumption.
\begin{assumption}   \label{asump:p}
  There are constants $c_0, h_0>0$ such that when $h \in (0,
  h_0)$, we have
  \begin{align*}
    & \int \sup_{\|g\| \leq h} |p_*(x) + \nabla p_*(x)^\top g - p_*(x +
    g)| d x \leq c_0 h^2 , \\
    & \int \frac{\sup_{\|g\| \leq h} |p_*(x+g)-p_*(x)|^2}{p_*(x)} d x \leq c_0  
      h^2 , \\
    & \int \|\nabla p_*(x)\| d x \leq c_0 .
  \end{align*}
\end{assumption}
The assumption says that $p_*$ is a smooth density function
with light tails. For example, common exponential distributions such as
Gaussian distributions and mixtures of Gaussians all satisfy the
assumption.  It is worth mentioning that the assumption is not truly needed
for the algorithm. This is because an arbitrary 
distribution can always be approximated to an arbitrary precision by mixtures of Gaussians.
Nevertheless, the assumption simplifies the statement of our analysis
(Theorem~\ref{thm:bgrad} below) because we do not have to deal
with such approximations. 

Also, to meet the assumption in practice, 
one can add a small Gaussian noise to every observed data point, as also noted 
by \cite{AB17}, but we empirically found that our method works without adding noise
on image generation, which is good as we have one fewer meta-parameters.  

\subsection{Theorem \ref{thm:bgrad} and its proof}

\begin{theorem}
Under the assumptions in Appendix \ref{app:assump}, 
let $g: \dataspace \to \dataspace$ be a continuously
differentiable transformation such that $\|g(\cdot)\| \leq a$ and 
$\|\nabla g(\cdot)\| \leq b$.
Let $\genp$ be the probability density of a random variable $\rvGen$, and 
let $\nextgenp$ be the probability density of the random variable $\rvNextgen$ 
such that $\rvNextgen=\rvGen+\eta g(\rvGen)$ where  $0 < \eta < \min(1/b,h_0/a)$.
Then there exists a
positive constant $c$ such that for all $\epsilon>0$:
\[
L_\beta(\nextgenp) \leq L_\beta(\genp) 
- \eta \int p_*(\any)  u(\any) \; g(\any)^\top\nabla \Db(\any)   \; d \any
+ c \eta^2  + c \eta \epsilon ,
\]
where $u(\any) = \beta -(1-\beta) \exp(\Db(\any))$.  
\label{thm:bgrad}
\end{theorem}

\newcommand{\refThm}{\ref{thm:bgrad}} %
\newcommand{\refAsumpp}{\ref{asump:p}} %

\paragraph{Notation}
We use $\|\cdot\|$ to denote the vector 2-norm 
and the matrix spectral norm (the largest singular value of a matrix).
Given a differentiable scalar function $h(\justx): \dataspace \to \Real$,
we use $\nabla h(\justx)$ to denote its gradient, which becomes a
$\datadim$-dimensional vector function. 
Given a differentiable function $g(\justx): \dataspace \to \dataspace$, we use
$\nabla g(\justx)$ to denote its Jacobi matrix and we use
$\nabla \cdot g(\justx)$ to denote the divergence of $g(\justx)$, defined as
\[
\nabla \cdot g(\justx) := \sum_{j=1}^\datadim \frac{\partial
  g(\justx)}{\partial [\justx]_j} ,
\]
where we use $[\justx]_j$ to denote the $j$-th component of $\justx$.  
We know that
\begin{equation}
\int \nabla \cdot w(x) d x = 0 \label{eq:proof-int-div}
\end{equation}
for all vector function $w(x)$ such that $w(\infty)=0$.

\begin{lemma}
Assume that $g(\gen): \dataspace \to \dataspace$ is a  continuously
differentiable transformation. Assume that $\|g(\gen)\| \leq a$
and $\|\nabla g(\gen)\| \leq b$, 
then as $\eta b <1$, the inverse
transformation $\gen = f^{-1}(\nextgen)$ of
$\nextgen = f(\gen)=\gen+\eta g(\gen)$ is unique.

Moreover, consider transformation of random variables by $f^{-1}(\cdot)$.  
Define $\tilde{p}_*$ to be the associated probability density function after this transformation 
when the pdf before the transformation is $p_*$.  
Then for any $\any \in \dataspace$, we have: 
\begin{equation}
\tilde{p}_*(\any) = p_*(f(\any)) |\det (\nabla f(\any))| .
\label{eq:proof-dens}
\end{equation}
Similarly, we have
\begin{equation}
\genp(\any) = \nextgenp(f(\any)) |\det (\nabla f(\any))| , 
\label{eq:proof-dens2}
\end{equation}
\label{lem:inverse}
where $\genp$ and $\nextgenp$ are defined in Theorem \ref{thm:bgrad}.  
\end{lemma}
\begin{proof}
Given $\nextgen$, define map $g'(\justx)$ as $g'(\justx) = \nextgen - \eta g(\justx)$, then the
assumption implies that $g'(\justx)$ is a contraction when $\eta b < 1$:
$\|g'(\justx) - g'(\justx')\| \leq \eta b \|\justx-\justx'\|$.
Therefore $g'(\justx)$ has a unique fixed 
point $\gen$, which leads to the inverse transformation
$f^{-1}(\nextgen)=\gen$.

\eqref{eq:proof-dens} and \eqref{eq:proof-dens2} follow
from the standard density formula
under transformation of variables.
\end{proof}

\begin{lemma}
Under the assumptions of Lemma~\ref{lem:inverse}, 
there exists a constant $c>0$ such that
\begin{equation}
|\det (\nabla f(\gen)) - (1 + \eta \nabla \cdot g(\gen))|
\leq c \eta^2 . \label{eq:proof-det}
\end{equation}
\end{lemma}
\begin{proof}

We note that
\[
\nabla f(\gen) = I + \eta \nabla g(\gen) .
\]
Therefore 
\[
\det (\nabla f(\gen)) = 1 + \eta \nabla \cdot g(\gen)
+ \sum_{j \geq 2} \eta^j m_j(g(\gen) ) ,
\]
where $m_j(g)$ is a function of $\nabla g$. Since $\nabla g$ is bounded, we obtain
the desired formula. 
\end{proof}

\begin{lemma}
Under the assumptions of Lemma~\ref{lem:inverse}, and assume that
Assumption~\refAsumpp\ holds, then there exists a constant $c>0$
such that
\begin{equation}
\int \big|\tilde{p}_*(\gen) - (p_*(\gen) +\eta
p_*(\gen) \nabla \cdot g(\gen) + \eta \nabla p_*(\gen)^\top
g(\gen))  \big| d \gen  \leq c \eta^2 . \label{eq:proof-delta-p}
\end{equation}
and
\begin{equation}
\int  \frac{(\tilde{p}_*(\gen)-p_*(\gen))^2}{{p}_*(\gen)}
d \gen 
\leq c \eta^2 . \label{eq:proof-delta-p2}
\end{equation}
\end{lemma}
\begin{proof}
Using the algebraic inequality
\begin{align*}
& \big|p_*(f(\gen)) |\det (\nabla
f(\gen))| - (p_*(\gen) +\eta
p_*(\gen) \nabla \cdot g(\gen) + \eta \nabla p_*(\gen)^\top
g(\gen))  \big|\\
\leq&
\big|p_*(f(\gen)) - (p_*(\gen) + \eta \nabla p_*(\gen)^\top
g(\gen)) \big| \;  \big|\det (\nabla f(\gen))\big| \\
& +
\big| (p_*(\gen) + \eta \nabla p_*(\gen)^\top
g(\gen)) \big| \; \big|(1 + \eta \nabla \cdot g(\gen))-|\det
  (\nabla f(\gen))|\big| \\
& + \eta^2 \big| \nabla \cdot g(\gen) \; \nabla p_*(\gen)^\top
g(\gen))  \big| ,
\end{align*}
and using  $\tilde{p}_*(\gen) = p_*(f(\gen)) |\det (\nabla
f(\gen))|$
from \eqref{eq:proof-dens}, we obtain
\begin{align*}
&\int \big|\tilde{p}_*(\gen) - (p_*(\gen) +\eta
p_*(\gen) \nabla \cdot g(\gen) + \eta \nabla p_*(\gen)^\top
g(\gen))  \big| d \gen \\
\leq& 
\underbrace{\int \big|p_*(f(\gen)) - (p_*(\gen) + \eta \nabla p_*(\gen)^\top
g(\gen)) \big| \;  |\det (\nabla f(\gen))| d \gen}_{A_0} \\
& + \underbrace{\int \big| (p_*(\gen) + \eta \nabla p_*(\gen)^\top
g(\gen)) \big| \; \big|(1 + \eta \nabla \cdot g(\gen))-|\det
  (\nabla f(\gen))|\big|  \;   d \gen}_{B_0} \\
& + \eta^2 \underbrace{\int \big| \nabla \cdot g(\gen) \; \nabla p_*(\gen)^\top
g(\gen))  \big| d \gen}_{C_0} \\
\leq & c \eta^2
\end{align*}
for some constant $c>0$, which proves \eqref{eq:proof-delta-p}. The last inequality uses the following facts.
\[
A_0 = \int \big|p_*(f(\gen)) - (p_*(\gen) + \eta \nabla p_*(\gen)^\top
g(\gen)) \big| \; O(1)  d \gen = O(\eta^2) ,
\]
where the first equality follows from the boundedness of $g$ and
$\nabla g$, and the second equality follows from the first inequality of Assumption~\refAsumpp. 
\[
B_0 = \int \big| (p_*(\gen) + \eta \nabla p_*(\gen)^\top
g(\gen)) \big| \; O(\eta^2) \;  d \gen
= O(\eta^2) ,
\]
where the first equality follows from \eqref{eq:proof-det}, and the
second equality follows from the third equality of Assumption~\refAsumpp. 
\[
C_0 =\int \|\nabla p_*(\gen)\| O(1) d \gen = O(1) ,
\]
where the first equality follows from the boundedness of $g$ and
$\nabla g$, and the second equality follows from the third equality of Assumption~\refAsumpp. 

Moreover, using \eqref{eq:proof-dens}, we obtain
\[
|\tilde{p}_*(\gen)-p_*(\gen)|
\leq 
|{p}_*(f(\gen))-p_*(\gen)| \; |\det (\nabla f(\gen))|
+ 
p_*(\gen) ||\det (\nabla f(\gen))|-1| .
\]
Therefore 
\begin{align*}
& \int  \frac{(\tilde{p}_*(\gen)-p_*(\gen))^2}{{p}_*(\gen)}
d \gen \\
\leq & 
2 \int 
      \frac{({p}_*(f(\gen))-p_*(\gen))^2|\det (\nabla f(\gen))|^2 + p_*(\gen)^2 (|\det (\nabla f(\gen))|-1)^2}{{p}_*(\gen)} 
d \gen  \leq c \eta^2 
\end{align*}
for some $c>0$, which proves \eqref{eq:proof-delta-p2}. The second inequality follows from the second
inequality of Assumption~\refAsumpp, and the boundedness of  $|\det
(\nabla f(\gen))|$, and the fact that 
$||\det (\nabla f(\gen))|-1|= O(\eta)$ from \eqref{eq:proof-det}.
\end{proof}

\subsection*{Proof of Theorem~\refThm}
\newcommand{\pp}{\genp} %
In the following integration, with a change of variable from $\gen$ to $\nextgen$ using $\nextgen=f(\gen)$ as
in Lemma~\ref{lem:inverse}, we obtain
\begin{align*}
&\int (\beta p_*(\nextgen)+(1-\beta)\nextgenp(\nextgen)) \ln \frac{\beta
        p_*(\nextgen)+(1-\beta)\nextgenp(\nextgen)}{(1-\beta)p_*(\nextgen)+\beta
   \nextgenp(\nextgen)} d \nextgen \\
=&\int (\beta p_*(f(\gen))+(1-\beta)p'(f(\gen))) \ln \frac{\beta
        p_*(f(\gen))+(1-\beta)p'(f(\gen))}{(1-\beta)p_*(f(\gen))+\beta
   p'(f(\gen))} |\det(\nabla f(x))| d \gen \\
&= \int (\beta \tilde{p}_*(\gen)+(1-\beta)\pp(\gen)) \ln \frac{\beta
  \tilde{p}_*(\gen)+(1-\beta)\pp(\gen)}{(1-\beta)\tilde{p}_*(\gen)+\beta\pp(\gen)}
  d \gen ,
\end{align*}
where the first inequality is basic calculus, and the second
inequality uses \eqref{eq:proof-dens} and \eqref{eq:proof-dens2}.

It follows that
\begin{align*}
L_\beta(\nextgenp) =&\int (\beta p_*(\nextgen)+(1-\beta)\nextgenp(\nextgen)) \ln \frac{\beta
        p_*(\nextgen)+(1-\beta)\nextgenp(\nextgen)}{(1-\beta)p_*(\nextgen)+\beta \nextgenp(\nextgen)} d \nextgen\\             
=& \int (\beta \tilde{p}_*(\gen)+(1-\beta)\pp(\gen)) \ln \frac{\beta
  \tilde{p}_*(\gen)+(1-\beta)\pp(\gen)}{(1-\beta)\tilde{p}_*(\gen)+\beta\pp(\gen)}
  d \gen \\
=& A_1 + B_1 + C_1 ,
\end{align*}
where $A_1$, $B_1$, and $C_1$ are defined as follows. 
\begin{align*}
A_1=&\int (\beta\tilde{p}_*(\gen)+(1-\beta)\pp(\gen)) \ln \frac{\beta
  {p}_*(\gen)+(1-\beta)\pp(\gen)}{(1-\beta){p}_*(\gen)+\beta\pp(\gen)}
  d \gen  \\
=&\int (\beta  {p}_*(\gen)+(1-\beta)\pp(\gen)) \ln \frac{\beta
  {p}_*(\gen)+(1-\beta)\pp(\gen)}{(1-\beta){p}_*(\gen)+\beta\pp(\gen)}
  d \gen \\
& + \eta \beta \int (
p_*(\gen) \nabla \cdot g(\gen) + \nabla p_*(\gen)^\top
g(\gen)) \ln \frac{\beta
  {p}_*(\gen)+(1-\beta)\pp(\gen)}{(1-\beta) {p}_*(\gen)+\beta\pp(\gen)}
  d \gen + O(\eta^2)\\
=&  L_\beta(\pp) 
+ \beta \eta \int 
\nabla \cdot (p_*(\gen) g(\gen)) \; \idealDb(\gen)
  d \gen + O(\eta^2) \\
=&  L_\beta(\pp) 
+ \beta \eta \int 
\nabla \cdot (p_*(\gen) g(\gen)) \; \Db(\gen)
  d \gen + O(\eta\epsilon+\eta^2) \\
=&  L_\beta(\pp) 
-\beta \eta \int 
p_*(\gen) g(\gen)^\top \nabla \Db(\gen)
  d \gen + O(\eta\epsilon+\eta^2) ,
\end{align*}
where the second equality uses \eqref{eq:proof-delta-p} and the fact
that $B<\infty$ in the statement of the theorem. The fourth
equality uses the $\epsilon$-approximation condition 
\eqref{eq:eps} of the assumption.    
The last equality uses
integration by parts and \eqref{eq:proof-int-div}.

\begin{align*}
B_1=&\int (\beta\tilde{p}_*(\gen)+(1-\beta)\pp(\gen)) \ln
      \frac{\beta\tilde{p}_*(\gen)+(1-\beta)\pp(\gen)}
{\beta {p}_*(\gen)+(1-\beta)\pp(\gen)} d \gen \\
  = &  \int (\beta\tilde{p}_*(\gen)+(1-\beta)\pp(\gen)) \ln \left(
      1+ \beta
      \frac{\tilde{p}_*(\gen)-{p}_*(\gen)}{\beta p_*(\gen)+(1-\beta)\pp(\gen)}\right)
      d \gen\\
\leq& \beta \int (\beta\tilde{p}_*(\gen)+(1-\beta)\pp(\gen)) 
      \frac{\tilde{p}_*(\gen)-{p}_*(\gen)}{\beta p_*(\gen)+(1-\beta)\pp(\gen)}
      d \gen \\
=&\beta^2 \int
   \frac{(\tilde{p}_*(\gen)-{p}_*(\gen))^2}{\beta p_*(\gen)+(1-\beta)\pp(\gen)}
      d \gen = O(\eta^2) ,
\end{align*}
where the  inequality uses $\ln (1+ \delta) \leq \delta$. The last
equality uses \eqref{eq:proof-delta-p2}.

\begin{align*}
C_1=&\int (\beta\tilde{p}_*(\gen)+(1-\beta)\pp(\gen)) \ln
      \frac{(1-\beta){p}_*(\gen)+\beta\pp(\gen)}
{(1-\beta)\tilde{p}_*(\gen)+\beta\pp(\gen)} d \gen\\
=& \int (\beta\tilde{p}_*(\gen)+(1-\beta)\pp(\gen)) \ln \left( 1 +(1-\beta)
      \frac{{p}_*(\gen)-\tilde{p}_*(\gen)}{(1-\beta)\tilde{p}_*(\gen)+\beta\pp(\gen)} \right)
d \gen\\
\leq&  (1-\beta) \int (\beta\tilde{p}_*(\gen)+(1-\beta)\pp(\gen)) 
      \frac{{p}_*(\gen)-\tilde{p}_*(\gen)}{(1-\beta)\tilde{p}_*(\gen)+\beta\pp(\gen)}
      d \gen\\
=&  (1-\beta) \int (\beta\tilde{p}_*(\gen)+(1-\beta)\pp(\gen)) 
      \frac{{p}_*(\gen)-\tilde{p}_*(\gen)}{(1-\beta){p}_*(\gen)+\beta\pp(\gen)} 
d \gen\\
& + (1-\beta)^2 \int (\beta\tilde{p}_*(\gen)+(1-\beta)\pp(\gen)) 
      \frac{({p}_*(\gen)-\tilde{p}_*(\gen))^2}{((1-\beta)\tilde{p}_*(\gen)+\beta\pp(\gen))
  ((1-\beta){p}_*(\gen)+\beta\pp(\gen))} 
d \gen \\
\underset{(a)}{=} &  (1-\beta) \int (\beta\tilde{p}_*(\gen)+(1-\beta)\pp(\gen)) 
      \frac{{p}_*(\gen)-\tilde{p}_*(\gen)}{(1-\beta){p}_*(\gen)+\beta\pp(\gen)} 
d \gen + O(\eta^2) \\
\underset{(b)}{=}& -(1-\beta) \int (\beta p_*(\gen)+(1-\beta)\pp(\gen)) 
      \frac{\eta
p_*(\gen) \nabla \cdot g(\gen) + \eta \nabla p_*(\gen)^\top
g(\gen)}{(1-\beta){p}_*(\gen)+\beta\pp(\gen)} 
d \gen
+ O(\eta^2) \\
=& -(1-\beta) \eta \int 
(p_*(\gen) \nabla \cdot g(\gen) + \nabla  p_*(\gen)^\top g(\gen))
\exp(\idealDb(\gen)) d \gen
+ O(\eta^2) \\
\underset{(c)}{=}& -(1-\beta) \eta \int 
(p_*(\gen) \nabla \cdot g(\gen) + \nabla  p_*(\gen)^\top g(\gen))
\exp(\Db(\gen)) d \gen
+ O(\eta\epsilon+\eta^2) \\
=& - (1-\beta) \eta \int (\nabla \cdot (p_*(\gen) g(\gen)))
  \exp(\Db(\gen)) d \gen
+ O(\eta \epsilon+\eta^2) \\
=& \eta (1-\beta) \int p_*(\gen) g(\gen)^\top
\nabla  \exp(\Db(\gen)) d \gen
+ O(\eta \epsilon+\eta^2) ,
\end{align*}
where the first inequality uses $\ln (1+\delta) \leq \delta$. The
equality (a) uses \eqref{eq:proof-delta-p2}. The equality
(b) uses \eqref{eq:proof-delta-p}. The equality (c) uses 
the $\epsilon$-approximation condition 
\eqref{eq:eps}. 
The last equality uses
integration by parts and \eqref{eq:proof-int-div}.

By combining the estimates of $A_1$, $B_1$, and $C_1$, we obtain the desired bound.

\clearpage
\section{Image examples}
\def \noshowBrand {}
\def \noshowTrand {}

We have shown random samples of generated images, and 
here we take a more focused approach.  We generate 1000 images by \xourAlg\ 
trained 
in the settings of Figure \ref{fig:conv-i} 
and show (roughly) the best and worst images among them.  
Similar to the inception score, 
the `goodness' of images are 
measured by the confidence of a classifier, e.g., 
a image that a classifier assigns a high probability of being a ``bedroom''
is considered to be a good bedroom image.   
The worst images are those with the highest entropy values. 
In Figures \ref{fig:m-clsall}--\ref{fig:t-ent}, we compare real images and generated images side by side 
that were chosen by the same procedure from a random sample of 1000 real images or 
1000 generated images (generated from one sequence of random inputs), respectively.  

%

%

%
\ifx \noshowMSsort \undefined
\newcommand{\mrclssz}{10x5}
\newcommand{\mclssz}{30x5}
\newcommand{\mrclsln}{0.25}
\newcommand{\mclsln}{0.75}
\newcommand{\mrclswd}{1.65in}
\newcommand{\mclswd}{5in}
\begin{figure*}
\centering
\begin{subfigure}[b]{\mrclsln\linewidth}
\begin{center}
\includegraphics[width=\mrclswd]{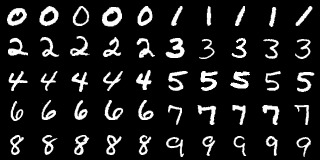}
\end{center}
\vskip -0.125in
\caption{\label{fig:m-clsall-real} \capfont
Real images
}
\end{subfigure}%
\begin{subfigure}[b]{\mclsln\linewidth}
\begin{center}
\includegraphics[width=\mclswd]{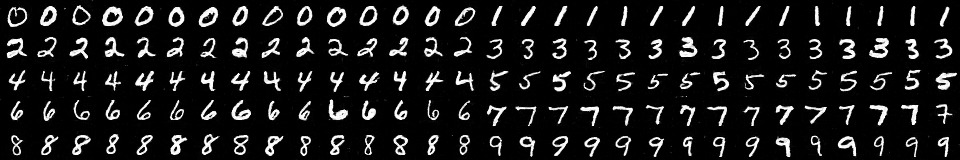}
\end{center}
\vskip -0.125in
\caption{\label{fig:m-clsall-xcfg} \capfont
Generated by \xourAlg\ (convolutional)
}
\end{subfigure}%
\vskip -0.15in
\caption{\label{fig:m-clsall} \capfont
`Best' digits.  
MNIST.  
For each digit, showing images with the highest probabilities 
among 1000 images that were either 
(a) randomly chosen from real data or (b) generated by \xourAlg. 
}
\vskip 0.2in
\renewcommand{\mrclssz}{10x3}
\renewcommand{\mclssz}{30x3}
\begin{subfigure}[b]{\mrclsln\linewidth}
\begin{center}
\includegraphics[width=\mrclswd]{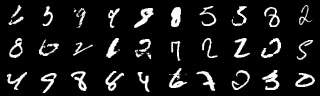}
\end{center}
\vskip -0.125in
\caption{\label{fig:m-ent-real} \capfont
Real images
}
\end{subfigure}%
\begin{subfigure}[b]{\mclsln\linewidth}
\begin{center}
\includegraphics[width=\mclswd]{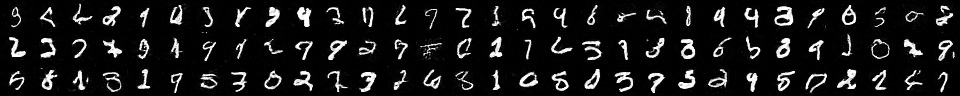}
\end{center}
\vskip -0.125in
\caption{\label{fig:m-ent-xcfg} \capfont
Generated by \xourAlg\ (convolutional)
}
\end{subfigure}%
\vskip -0.15in
\caption{\label{fig:m-ent} \capfont
`Worst' digits.  MNIST. 
Images with the highest entropy among 1000 images that were either 
(a) randomly chosen from real data or 
(b) generated by \xourAlg.  
Some of the generated images in (b) are hard to tell what digits they are, but 
so are some of the real images in (a).  
}
\newcommand{\srclssz}{8x5}
\newcommand{\sclssz}{24x5}
\newcommand{\srclsln}{0.25}
\newcommand{\sclsln}{0.75}
\newcommand{\srclswd}{1.65in}
\newcommand{\sclswd}{5in}
\vskip 0.2in
\centering
\begin{subfigure}[b]{\srclsln\linewidth}
\begin{center}
\includegraphics[width=\srclswd]{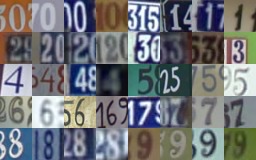}
\end{center}
\vskip -0.125in
\caption{\label{fig:s-clsall-real} \capfont
Real images
}
\end{subfigure}%
\begin{subfigure}[b]{\sclsln\linewidth}
\begin{center}
\includegraphics[width=\sclswd]{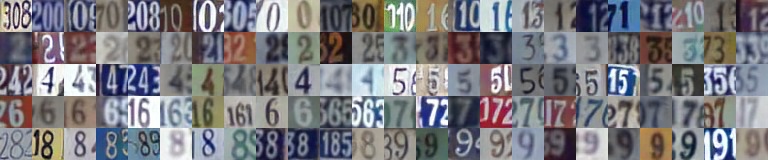}
\end{center}
\vskip -0.125in
\caption{\label{fig:s-clsall-xcfg} \capfont
Generated by \xourAlg\ (convolutional) 
}
\end{subfigure}%
\vskip -0.15in
\caption{\label{fig:s-clsall} \capfont
`Best' digits.  SVHN. 
For each digit, showing images with the highest probabilities 
among 1000 images that were eigher 
(a) randomly chosen from real data or (b) generated by \xourAlg. 
}
\vskip 0.2in
\renewcommand{\srclssz}{8x3}
\renewcommand{\sclssz}{24x3}
\begin{subfigure}[b]{\srclsln\linewidth}
\begin{center}
\includegraphics[width=\srclswd]{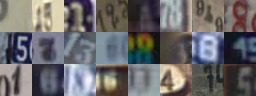}
\end{center}
\vskip -0.125in
\caption{\label{fig:s-ent-real} \capfont
Real images
}
\end{subfigure}%
\begin{subfigure}[b]{\sclsln\linewidth}
\begin{center}
\includegraphics[width=\sclswd]{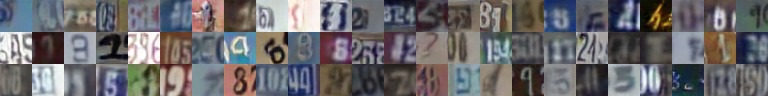}
\end{center}
\vskip -0.125in
\caption{\label{fig:s-ent-xcfg} \capfont
Generated by \xourAlg\ (convolutional) 
}
\end{subfigure}%
\vskip -0.15in
\caption{\label{fig:s-ent} \capfont
`Worst' digits.  SVHN.   
Images with the highest entropy among 1000 images that were either 
(a) randomly chosen from real data or 
(b) generated by \xourAlg.  
Some of the generated images in (b) are hard to tell what digits they are, but 
so are some of the real images in (a). 
}
\end{figure*}
\fi %
\newcommand{\morehwd}{3in}
\newcommand{\morewd}{6.6in}
\newcommand{\moresz}{11x5}
\newcommand{\morersz}{11x1}
\ifx \noshowBrand \undefined 
\begin{figure*}
\centering
\begin{subfigure}[b]{1\linewidth}
\begin{center}
\includegraphics[width=\morewd]{pbrlrcr-real-\morersz.jpg}
\vskip -0.06in
\caption{\label{fig:bm-real-img} 
Real images in the training set.  
}
\end{center}
\end{subfigure}%
\quad
\vskip 0.05in
\begin{subfigure}[b]{1\linewidth}
\begin{center}
\includegraphics[width=\morewd]{pbrlrcr-xcfg-cv-\moresz-\bciik.jpg}
\vskip -0.06in
\caption{\label{fig:bm-xcfg-c-img} 
Generated by \xourAlg\ (4-block ResNet)
}
\end{center}
\end{subfigure}%
\quad
\vskip 0.05in
\begin{subfigure}[b]{1\linewidth}
\begin{center}
\includegraphics[width=\morewd]{pbrlrcr-gp-cv-\moresz-\bciik.jpg}
\vskip -0.06in
\caption{\label{fig:bm-best-c-img} 
Generated by \wgangp\ (4-block ResNet)
}
\end{center}
\end{subfigure}%
\quad
\vskip 0.05in
\begin{subfigure}[b]{0.47\linewidth}
\begin{center}
\includegraphics[width=\morehwd]{pbrlrcr-g0-cv-5x3-\bciik.jpg}
\vskip -0.06in
\caption{\label{fig:bm-c-g0-img} 
Generated by \gano\ (4-block ResNet)
}
\end{center}
\end{subfigure}%
\begin{subfigure}[b]{0.47\linewidth}
\begin{center}
\includegraphics[width=\morehwd]{pbrlrcr-g1-cv-5x3-\bciik.jpg}
\vskip -0.06in
\caption{\label{fig:bm-c-g1-img} 
Generated by \gand\ (4-block ResNet)
}
\end{center}
\end{subfigure}%
\vskip -0.125in
\caption{\label{fig:b-moremore} \capfont
Example images.  
LSUN bedrooms \& living rooms.
Randomly chosen and sorted by predicted classes. 
}
\end{figure*}
\fi %
\newcommand{\clssz}{9x4}
\newcommand{\rclssz}{2x4}

\newcommand{\clswd}{5.4in} %
\newcommand{\rclswd}{1.2in} %
\newcommand{\rclsln}{0.19}
\newcommand{\clsln}{0.81}

\ifx \noshowBsort \undefined
\clearpage
\begin{figure*}
\centering
\begin{subfigure}[b]{\rclsln\linewidth}
\begin{center}
\includegraphics[width=\rclswd]{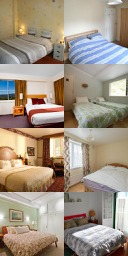}
\vskip -0.06in
\caption{\label{fig:b-real-cls0} 
Real images.
}
\end{center}
\end{subfigure}%
\begin{subfigure}[b]{\clsln\linewidth}
\begin{center}
\includegraphics[width=\clswd]{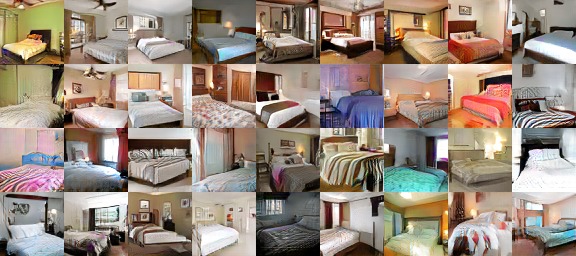}
\vskip -0.06in
\caption{\label{fig:b-xcfg-cls0} 
Generated by \xourAlg\ (4-block ResNet)
}
\end{center}
\end{subfigure}%
\vskip -0.2in
\caption{\label{fig:b-cls0} \capfont
Bedrooms `best' among 1000 (LSUN BR+LR).  Predicted by a classifier to be ``bedroom''  
with the highest probabilities among 1000 images that were either 
(a) randomly chosen from real data or (b) generated by \xourAlg. 
}
\end{figure*}
\begin{figure*}
\begin{subfigure}[b]{\rclsln\linewidth}
\begin{center}
\includegraphics[width=\rclswd]{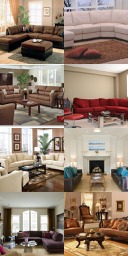}
\vskip -0.06in
\caption{\label{fig:b-real-cls1} 
Real images. 
}
\end{center}
\end{subfigure}%
\begin{subfigure}[b]{\clsln\linewidth}
\begin{center}
\includegraphics[width=\clswd]{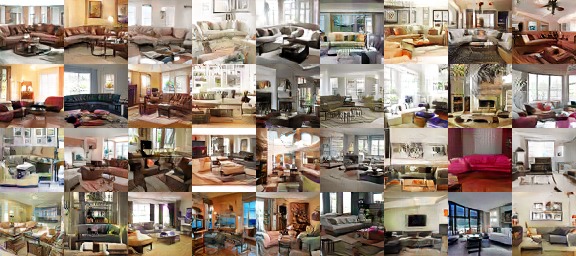}
\vskip -0.06in
\caption{\label{fig:b-xcfg-cls1} 
Generated by \xourAlg\ (4-block ResNet)
}
\end{center}
\end{subfigure}%
\vskip -0.2in
\caption{\label{fig:b-cls1} \capfont
Living rooms `best' among 1000 (LSUN BR+LR).  Predicted by a classifier to be ``living room''  
with the highest probabilities among 1000 images that were either 
(a) randomly chosen from real data or (b) generated by \xourAlg. 
}
\end{figure*}
\begin{figure*}
\begin{subfigure}[b]{\rclsln\linewidth}
\begin{center}
\includegraphics[width=\rclswd]{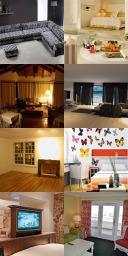}
\vskip -0.06in
\caption{\label{fig:b-real-ent} 
Real images. 
}
\end{center}
\end{subfigure}%
\begin{subfigure}[b]{\clsln\linewidth}
\begin{center}
\includegraphics[width=\clswd]{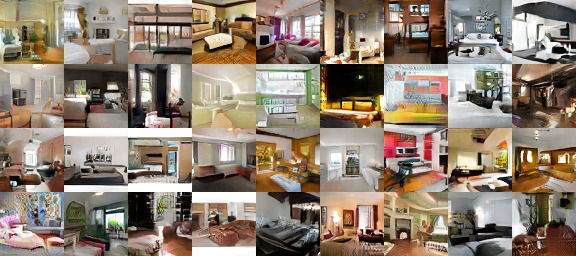}
\vskip -0.06in
\caption{\label{fig:b-xcfg-ent} 
Generated by \xourAlg\ (4-block ResNet)
}
\end{center}
\end{subfigure}%
\vskip -0.125in
\caption{\label{fig:b-ent} \capfont
Bedrooms/living rooms `worst' among 1000 (LSUN BR+LR).  
Images with the highest entropy among 1000 images that were either 
(a) randomly chosen from real data or (b) generated by \xourAlg. 
The generated images in (b) could be either of relatively low quality 
or depicting hard-to-tell rooms as the real images in (a) do.  
}
\end{figure*}
\fi %
\ifx \noshowTrand \undefined
\begin{figure*}
\centering
\begin{subfigure}[b]{1\linewidth}
\begin{center}
\includegraphics[width=\morewd]{ptwbgcr-real-\morersz.jpg}
\vskip -0.06in
\caption{\label{fig:tm-real-img} 
Real images in the training set.  
}
\end{center}
\end{subfigure}%
\quad
\vskip 0.05in
\begin{subfigure}[b]{1\linewidth}
\begin{center}
\includegraphics[width=\morewd]{ptwbgcr-xcfg-cv-\moresz-\bciik.jpg}
\vskip -0.06in
\caption{\label{fig:tm-xcfg-c-img} 
Generated by \xourAlg\ (4-block ResNet)
}
\end{center}
\end{subfigure}%
\quad
\vskip 0.05in
\begin{subfigure}[b]{1\linewidth}
\begin{center}
\includegraphics[width=\morewd]{ptwbgcr-gp-cv-\moresz-\bciik.jpg}
\vskip -0.06in
\caption{\label{fig:tm-best-c-img} 
Generated by \wgangp\ (4-block ResNet)
}
\end{center}
\end{subfigure}%
\quad
\vskip 0.05in
\begin{subfigure}[b]{0.47\linewidth}
\begin{center}
\includegraphics[width=\morehwd]{ptwbgcr-g0-cv-5x3-\bciik.jpg}
\vskip -0.06in
\caption{\label{fig:tm-best-c-img} 
Generated by \gano\ (4-block ResNet)
}
\end{center}
\end{subfigure}%
\begin{subfigure}[b]{0.47\linewidth}
\begin{center}
\includegraphics[width=\morehwd]{ptwbgcr-g1-cv-5x3-\bciik.jpg}
\vskip -0.06in
\caption{\label{fig:tm-best-c-img} 
Generated by \gand\ (4-block ResNet)
}
\end{center}
\end{subfigure}%
\vskip -0.125in
\caption{\label{fig:t-moremore} \capfont
Example images.  
LSUN towers \& bridges.
Randomly chosen and sorted by predicted classes. 
}
\end{figure*}
\fi %

\ifx \noshowTsort \undefined
\begin{figure*}
\centering
\begin{subfigure}[b]{\rclsln\linewidth}
\begin{center}
\includegraphics[width=\rclswd]{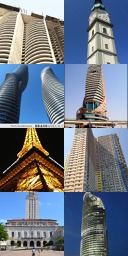}
\vskip -0.06in
\caption{\label{fig:t-real-cls0} 
Real images.
}
\end{center}
\end{subfigure}%
\begin{subfigure}[b]{\clsln\linewidth}
\begin{center}
\includegraphics[width=\clswd]{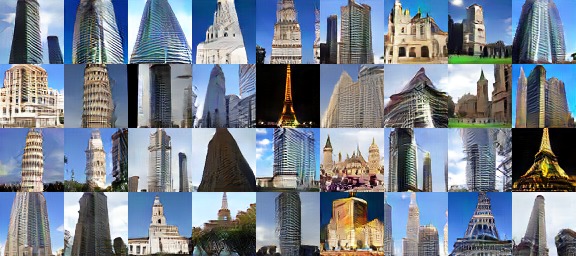}
\vskip -0.06in
\caption{\label{fig:t-xcfg-cls0} 
Generated by \xourAlg\ (4-block ResNet)
}
\end{center}
\end{subfigure}%
\vskip -0.18in
\caption{\label{fig:t-cls0} \capfont
Towers `best' among 1000 (LSUN T+B).  Predicted by a classifier to be ``tower''  
with the highest probabilities among 1000 images that were either 
(a) randomly chosen from real data or (b) generated by \xourAlg. 
}
\end{figure*}
\begin{figure*}
\begin{subfigure}[b]{\rclsln\linewidth}
\begin{center}
\includegraphics[width=\rclswd]{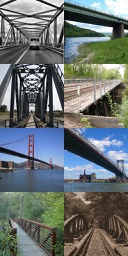}
\vskip -0.06in
\caption{\label{fig:t-real-cls1} 
Real images. 
}
\end{center}
\end{subfigure}%
\begin{subfigure}[b]{\clsln\linewidth}
\begin{center}
\includegraphics[width=\clswd]{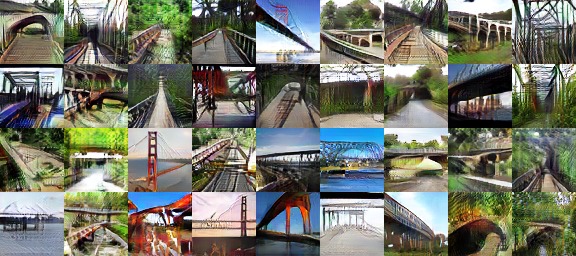}
\vskip -0.06in
\caption{\label{fig:t-xcfg-cls1} 
Generated by \xourAlg\ (4-block ResNet)
}
\end{center}
\end{subfigure}%
\vskip -0.18in
\caption{\label{fig:t-cls1} \capfont
Bridges `best' among 1000 (LSUN T+B).  Predicted by a classifier to be ``bridge''  
with the highest probabilities among 1000 images that were either 
(a) randomly chosen from real data or (b) generated by \xourAlg. 
}
\end{figure*}
\begin{figure*}
\begin{subfigure}[b]{\rclsln\linewidth}
\begin{center}
\includegraphics[width=\rclswd]{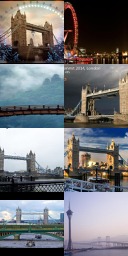}
\vskip -0.06in
\caption{\label{fig:t-real-ent} 
Real images. 
}
\end{center}
\end{subfigure}%
\begin{subfigure}[b]{\clsln\linewidth}
\begin{center}
\includegraphics[width=\clswd]{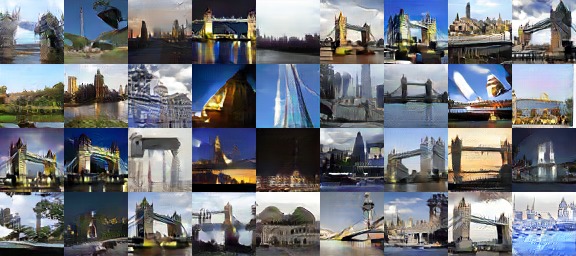}
\vskip -0.06in
\caption{\label{fig:t-xcfg-ent} 
Generated by \xourAlg\ (4-block ResNet)
}
\end{center}
\end{subfigure}%
\vskip -0.18in
\caption{\label{fig:t-ent} \capfont
Towers/bridges `worst' among 1000 (LSUN T+B).  
Images with the highest entropy among 1000 images that were either 
(a) randomly chosen from real data or (b) generated by \xourAlg. 
The generated images in (b) could be either of relatively low quality 
or depicting hard-to-tell objects as the real images in (a) do. 
}
\end{figure*}
\fi %
\clearpage

\section{Details of experimental setup}
\subsection{Network architectures} 

The network definitions below include batch normalization layers.  
Note that to experiment with \wgangp, 
we explored the options with the batch normalization layers being removed or partially removed, 
as described in Section \refsecsetup.  

\label{app:net}
\subsubsection{MNIST and SVHN in Figure \reffigconv} 

The convolutional architectures used for MNIST and SVHN are an extension of DCGAN, 
inserting 1$\times$1 convolution layers.   

\begin{center}
\begin{tabular}{|c|l|c|c|l|}
\multicolumn{2}{c}{Approximator/Generator}  &\multicolumn{1}{c}{}&   \multicolumn{2}{c}{Discriminator}\\
\cline{1-2} \cline{4-5}
1 & Projection                              && 1&Convolution, 5$\times$5, stride 2 \\
\cline{1-2}
2 & ReLU                                    && 2&LeakyReLU  \\
                                            \cline{4-5}
3 & Transposed conv, 5$\times$5, stride 2   && 3&Convolution, 5$\times$5, stride 2 \\
4 & BatchNorm                               && 4&BatchNorm  \\
5 & ReLU                                    && 5&LeakyReLU  \\
6 & Convolution, 1$\times$1, stride 1       && 6&Convolution, 1$\times$1, stride 1 \\
7 & BatchNorm                               && 7&BatchNorm \\
\cline{1-2} 
8 & ReLU                                    && 8&LeakyReLU \\
                                            \cline{4-5}
9 & Transposed conv, 5$\times$5, stride 2   && 9&Flatten \\
10 & $\tanh$                                &&10&Linear \\
\cline{1-2}                                 \cline{4-5}
\multicolumn{2}{l}{Repeat 2--7 twice.} &\multicolumn{1}{c}{}& \multicolumn{2}{l}{Repeat 3--8 twice.} \\
\end{tabular}
\end{center}

For the discriminator, start with 32 (MNIST) or 64 (SVHN) feature maps and double it at downsampling 
except for the first downsampling. 
For the approximator/generator, start with 128 (MNIST) or 256 (SVHN) and halve it 
at upsampling except for the last upsampling. 
\subsubsection{LSUN in Figure \reffigconv}
The convolutional architecture used for LSUN is a simplification of 
a residual network found at \url{https://github.com/igul222/improved_wgan_training}, 
reducing the number of batch normalization layers for speedup, 
removing some irregularity, 
and so forth. 
Both the approximator/generator and discriminator are 
a residual network with four convolution blocks.  
\begin{center}
\begin{tabular}{|c|l|c|c|l|}
\multicolumn{2}{c}{Approximator/Generator}  &\multicolumn{1}{c}{}&   \multicolumn{2}{c}{Discriminator}\\
\cline{1-2}                           \cline{4-5}
1& Projection                         && 1& ReLU (omitted in the 1st block) \\
\cline{1-2}
2& ReLU                               && 2& Convolution, 3$\times$3, stride 1 \\
3& Upsampling ($\times$2), nearest    && 3& ReLU \\
4& Convolution, 3$\times$3, stride 1  && 4& Convolution, 3$\times$3, stride 1 \\
5& ReLU                               && 5& BatchNorm \\
6& Convolution, 3$\times$3, stride 1  && 6& Downsampling (/2), mean\\
                                      \cline{4-5}
7& BatchNorm                          && 7& ReLU      \\
  \cline{1-2}
8& ReLU                               && 8& Flatten \\                 
9& Convolution, 3$\times$3, stride 1  &&  9& Linear \\
                                      \cline{4-5}
10& $\tanh$                           &\multicolumn{3}{c}{Repeat 1--6 four times.}\\
\cline{1-2}
\multicolumn{3}{c}{Repeat 2--7 four times.}\\
\end{tabular}
\end{center}
2--7 of the approximator/generator and 1--6 of the discriminator are 
convolution blocks with a shortcut connecting the beginning to the end. 
In the approximator/generator, 
the numbers of feature maps are 512 (produced by the projection layer) 
and 256, 256, 128, 128, 64, 64, 64, and 64 (produced by the convolution layers). 
In the discriminator, 
the numbers of feature maps produced by the convolution layers are 
64, 64, 64, 128, 128, 256, 256, and 512.  

\subsection{Details of experimental settings for \xourAlg} 
The network weights were initialized by the Gaussian with mean 0 and standard deviation 0.01.  

The rmsprop learning rate for \xourAlg\ (for updating the discriminator and the approximator) 
was fixed to 0.0001 across all the datasets when the approximator was fully-connected.  
Although 0.0001 worked well also for the convolutional approximator cases, 
these cases turned out to tolerate and benefit from a more aggressive learning rate 
resulting in faster training; hence, it was fixed to 0.00025 across all the datasets. 
Additionally, if we keep training long enough, the discriminator may eventually overfit 
as also noted on \wgangp\ in \cite{WGANgp17}.  
It may be useful to reduce the learning rate (e.g., by multiplying 0.1) 
towards the end of training if the onset of discriminator overfit needs to be delayed. 

As shown in Table \ref{tab:meta}, the discriminator update frequency $\dupd$ was fixed to 1.  
However, it is worth mentioning that stable training has been observed with a relatively 
wide range of $\dupd$ including $\dupd$=25.  The choice of reporting the results with 
$\dupd$=1 was due to its pragmatic advantage -- training tends to be faster with a smaller 
$\dupd$ as it leads to more frequent updates of a generator.  

To choose $\eta$ used for generator update 
$G_{t+1}(\inpt) = G_t(\inpt) + \eta \nabla D(G_t(\inpt))$, 
we tried some of $\{ 0.1,0.25,0.5,1,2.5\}$ (not all as we tried only promising ones) 
for each configuration, 
following the meta-parameter selection protocol described in Section \ref{sec:setup}.  
Typically, multiple values were found to work well, 
and the table below shows the chosen values.  
\begin{center}
\begin{tabular}{|l|c|c|c|c|}
\hline
                     & MNIST & SVHN & BR+LR & T+B \\
\hline                   
 convolutional (Fig.\ref{fig:conv-i})       &   1   & 0.25 &   1   &  1  \\
 conv. no batch norm (Fig.\ref{fig:conv-i}) &  --   & 0.5  &  2.5  & --  \\
 fully-connected (Fig.\ref{fig:fc-i})       &  0.1  & 0.25 &  0.5  & 0.5 \\ 
\hline
\end{tabular}
\end{center} 
In general, similar to the SGD learning rate, 
a larger $\eta$ leads to faster training, but a too large value would break training.  
A too small value should be avoided since stable training requires a generator to make sufficient progress
before each approximator update.  

\fi
\end{document}